\documentclass{article} 
\usepackage{nips13submit_e,times}
\nipsfinalcopy
\usepackage{hyperref}
\usepackage{url}

\usepackage[authoryear,round]{natbib}
\usepackage{amsmath,amssymb,amsthm}
\usepackage{color}
\usepackage{graphicx}
\usepackage{verbatim}

\newtheorem{theorem}{Theorem}
\newtheorem{lemma}{Lemma}

\newtheorem{proposition}{Proposition}

\theoremstyle{definition}
\newtheorem{definition}{Definition}
\newtheorem{problem}{Problem}

\newtheorem{example}{Example}

\newcommand{\Acal}{\mathcal{A}}
\newcommand{\Fcal}{\mathcal{F}}
\newcommand{\Rcal}{\mathcal{R}}

\newcommand{\bn}{\boldsymbol{n}}

\usepackage{fancyvrb}

\newcommand{\bmx}[0]{\begin{bmatrix}}
\newcommand{\emx}[0]{\end{bmatrix}}

\newcommand{\vect}[1]{\mathbf{#1}}

\newcommand{\matr}[1]{\mathbf{#1}}

\newcommand{\vb}[0]{\vect{b}}
\newcommand{\vc}[0]{\vect{c}}
\newcommand{\vh}[0]{\vect{h}}

\newcommand{\vx}[0]{\vect{x}}
\newcommand{\vw}[0]{\vect{w}}
\newcommand{\vs}[0]{\vect{s}}
\newcommand{\vf}[0]{\vect{f}}
\newcommand{\vy}[0]{\vect{y}}
\newcommand{\vg}[0]{\vect{g}}

\newcommand{\mW}[0]{\matr{W}}

\newcommand{\mU}[0]{\matr{U}}
\newcommand{\mV}[0]{\matr{V}}

\newcommand{\RR}[0]{\mathbb{R}}

\newcommand{\Scal}[0]{\mathcal{S}}

\newcommand{\abs}[0]{\text{abs}}

\newcommand{\R}{\mathbb{R}}

\usepackage{mathtools}

\title{On the number of response regions of deep feedforward networks with piecewise linear activations}

\author{Razvan Pascanu\\
Universit\'e de Montr\'eal\\
Montr\'eal QC H3C 3J7 Canada\\
\texttt{r.pascanu@gmail.com} \\
\And
Guido Mont\'ufar \\
Max Planck Institute for Mathematics in the Sciences \\
Inselstra\ss e 22, 04103 Leipzig, Germany \\
\texttt{montufar@mis.mpg.de} \\
\AND
Yoshua Bengio\\
Universit\'e de Montr\'eal\\
Montr\'eal QC H3C 3J7 Canada\\
\texttt{yoshua.bengio@umontreal.ca}\\
}

\begin{document}

\maketitle

\begin{abstract}
This paper explores the complexity of deep feedforward networks with linear
pre-synaptic couplings and rectified linear activations.  This is a
contribution to the growing body of work contrasting the representational power
of deep and shallow network architectures. In particular, we offer a framework
for comparing deep and shallow models that belong to the family of  piecewise
linear functions based on computational geometry. We look at a deep rectifier
multi-layer perceptron (MLP) with linear outputs units and compare it with a
single layer version of the model.  In the asymptotic regime, when the number
of inputs stays constant, if the shallow model has $kn$ hidden units and $n_0$
inputs, then the number of linear regions is $O(k^{n_0}n^{n_0})$. For a $k$
layer model with $n$ hidden units on each layer it is $\Omega(\left\lfloor {n}/
{n_0}\right\rfloor^{k-1}n^{n_0})$. The number
$\left\lfloor{n}/{n_0}\right\rfloor^{k-1}$ grows faster than $k^{n_0}$ when $n$
tends to infinity or when $k$ tends to infinity and $n \geq 2n_0$.
Additionally, even when $k$ is small, if we restrict $n$ to be $2n_0$, we can
show that a deep model has considerably more linear regions that a shallow one.
We consider this as a first step towards understanding the complexity of these
models and specifically towards providing suitable mathematical tools for
future analysis. 

\end{abstract}

\section{Introduction}

Deep systems are believed to play an important role in information processing
of intelligent agents. A common hypothesis underlying this belief is that deep
models can be exponentially more efficient at representing some functions than
their shallow counterparts~\citep[see][]{bengio2009learning}. 

The argument is usually a compositional one. Higher layers in a deep model can
\textit{re-use} primitives constructed by the lower layers in order to build
gradually more complex functions.  For example, on a vision task, one would
hope that the first layer learns Gabor filters capable to detect edges of
different orientation.  These edges are then put together at the second layer
to form \textit{part-of-object} shapes. On higher layers, these part-of-object
shapes are combined further to obtain detectors for more complex part-of-object
shapes or objects. Such a behaviour is empirically illustrated, for instance,
in \citet{Zeiler+et+al-arxiv2013,HonglakL2009}.  On the other hand, a shallow
model has to construct detectors of target objects based only on the detectors
learnt by the first layer. 

The representational power of computational systems with shallow and deep
architectures  has been studied intensively.  A well known
result~\cite{Hajnal1993129} derived lower complexity bounds for shallow
threshold networks.  Other works have explored the representational power of
generative models based on Boltzmann machines
\citet{montufar2011,NIPS2013_5020} and deep belief networks
\citep{Sutskever:2008,LeRoux:2010,montufar2011refinements}, or have compared
mixtures and products of experts models \citep{montufar2012does}. 

In addition to such inspections, a wealth of evidence for the validity of this
hypothesis comes from deep models consistently outperforming shallow ones on a
variety of tasks and datasets \citep[see, e.g.,][]{Goodfellow_maxout_2013,
Hinton-et-al-arxiv2012, Hinton-et-al-2012}.  However, theoretical results on
the representational power of deep models are limited, usually due to the
composition of nonlinear functions in deep models, which makes mathematical
analysis difficult.  Up to now, theoretical results have focussed on circuit
operations (neural net unit computations) that are substantially different from
those being used in real state-of-the-art deep learning applications, such as
logic gates~\citep{Hastad86}, linear + threshold units with non-negative
weights~\citep{Hastad91} or polynomials~\citep{YoshuaOlivier}.
\citet{YoshuaOlivier} show that deep sum-product networks \citep{PoonDomingos}
can use exponentially less nodes to express some families of polynomials
compared to the shallow ones. 

The present note analyzes the representational power of deep MLPs with
rectifier units.  Rectifier units \citep{Glorot+al-AI-2011-small, Nair-2010}
and piecewise linearly activated units in general (like the \textit{maxout}
unit~\citep{Goodfellow_maxout_2013}), are becoming popular choices in designing
deep models, and most current state-of-the-art results involve using one of
such activations \citep{Goodfellow_maxout_2013,Hinton-et-al-arxiv2012}.
\citet{Glorot+al-AI-2011-small} show that rectifier units have several
properties that make the optimization problem easier than the more traditional
case using smooth and bounded activations, such as \textit{tanh} or
\textit{sigmoid}.

In this work we take advantage of the piecewise linear nature of the rectifier
unit to mathematically analyze the behaviour of deep rectifier MLPs.  Given
that the model is a composition of piecewise linear functions, it is itself a
piecewise linear function. We compare the flexibility of a deep model with that
of a shallow model by counting the number of linear regions they define over
the input space for a fixed number of hidden units. This is the number of
pieces available to the model in order to approximate some arbitrary nonlinear
function. For example, if we want to perfectly approximate some curved boundary
between two classes, a rectifier MLP will have to use infinitely many linear
regions.  In practice we have a finite number of pieces, and if we assume that
we can perfectly learn their optimal slopes, then the number of linear regions
becomes a good proxy for how well the model approximates this boundary.  In
this sense, the number of linear regions is an upper bound for the
flexibility of the model. In practice, the linear pieces are not independent
and the model may not be able to learn the right slope for each linear region.
Specifically, for deep models there is a correlation between regions, which
results from the sharing of parameters between the functions that describe the
output on each region. 

This is by no means a negative observation. If all the linear regions of the
deep model were independent of each other, by having many more linear regions,
deep models would grossly overfit. The correlation of the linear regions of a
deep model results in its ability to generalize, by allowing it to better
represent only a small family of structured functions. These are functions that
look complicated (e.g., a distribution with a huge number of modes) but that
have an underlying structure that the network can `compress' into its
parameters.  The number of regions, which indicates the number of variations
that the network can represent, provides a measure of how well it can fit this
family of structured functions (whose approximation potentially needs
infinitely many linear regions). 

We believe that this approach, based on counting the number of linear regions,
is extensible to any other piecewise linear activation function and also to other
architectures, including the \textit{maxout} activation and the convolutional
networks with rectifier activations. 

We know the maximal number of regions of linearity of functions computable by a
shallow model with a fixed number of hidden units.  This number is given by a
well studied geometrical problem. The main insight of the present work is to
provide a geometrical construction that describes the regions of linearity of
functions computed by deep models.  We show that in the asymptotic regime,
these functions have many more linear regions than the ones computed by shallow
models, for the same number of hidden units.

For the single layer case, each hidden unit divides the input space in two,
whereby the boundary is given by a hyperplane. For all input values on one side
of the hyperplane, the unit outputs a positive value. For all input values on
the other side of the hyperplane, the unit outputs $0$. Therefore, the question
that we are asking is: Into how many regions do $n$ hyperplanes split space?
This question is studied in geometry under the name of hyperplane arrangements,
with classic results such as Zaslavsky's theorem.
Section~\ref{section:onelayer} provides a quick introduction to the subject. 

For the multilayer version of the model we rely on the following intuition.  By
using the rectifier nonlinearity, we identify multiple regions of the input
space which are mapped by a given layer into an equivalent set of activations
and represent thus equivalent inputs for the next layers.  That is, a hidden
layer can perform a kind of \textit{or} operation by reacting similarly to
several different inputs.  Any subsequent computation made on these activations
is replicated on all equivalent inputs.

\smallskip 

This paper is organized as follows.  In Section~\ref{section:preliminaries} we
provide definitions and basic observations about piecewise linear functions.
In Section~\ref{section:onelayer} we discuss rectifier networks with one single
hidden layer and describe their properties in terms of hyperplane arrangements
which are fairly well known in the literature.  In
Section~\ref{section:klayers} we discuss deep rectifier networks and prove our
main result, Theorem~\ref{theorem:klayermodel}, which describes their
complexity in terms of the number of regions of linearity of functions that
they represent.  Details about the asymptotic behaviour of the results derived
in Sections~\ref{section:onelayer} and~\ref{section:klayers} are given in the
Appendix~\ref{section:asymptotic}.  In Section~\ref{sec:opt1} we analyze a
special type of deep rectifier MLP and show that even for a small number of
hidden layers it can generate a large number of linear regions.  In
Section~\ref{section:conclusions} we offer a discussion of the results. 

\section{Preliminaries}\label{section:preliminaries}

We consider classes of functions (models) defined in the following way. 

\begin{definition}
    A {\em rectifier feedforward network} is a layered feedforward network, or
    multilayer perceptron (MLP), as shown in Fig.~\ref{figure:nn}, with
    following properties.  Each hidden unit receives as inputs the real valued
    activations $x_1,\ldots,x_n$ of all units in the previous layer, computes
    the weighted sum $$s = \sum_{i\in[n]} w_{i} x_i + b, $$  and outputs the
    rectified value $$\text{rect}(s)=\max\{0,s\}.$$  The real parameters
    $w_1,\ldots, w_n$ are the {\em input weights} and $b$ is the {\em bias} of
    the unit.  The output layer is a {\em linear layer}, that is, the units in
    the last layer compute a linear combination of their inputs and output it
    unrectified.

    Given a vector of naturals $\boldsymbol{n}=(n_0, n_1,\ldots, n_L)$, we
    denote by $\Fcal_{\boldsymbol{n}}$ the set of all functions
    $\R^{n_0}\to\R^{n_L}$ that can be computed by a rectifier feedforward
    network with $n_0$ inputs and $n_l$ units in layer $l$ for $l\in[L]$. The
    elements of $\Fcal_{\boldsymbol{n}}$ are continuous piecewise linear
    functions. 

    We denote by $\Rcal(\bn)$ the maximum of the number of regions of linearity or {\em response regions} 
    over all functions from $\Fcal_{\bn}$.  For clarity, given a function
    $f\colon \R^{n_0}\to\R^{n_L}$, a connected open subset $R\subseteq\R^{n_0}$
    is called a {\em region of linearity} or {\em linear region} or {\em response region} of $f$ if the
    restriction $f|_{R}$ is a linear function and for any open set $\tilde
    R\supsetneq R$ the restriction $f|_{\tilde R}$ is not a linear function.
    In the next sections we will compute bounds on $\mathcal{R}(\bn)$ for
    different choices of $\bn$. We are especially interested in the comparison
    of shallow networks with one single very wide hidden layer and deep
    networks with many narrow hidden layers. 
\end{definition}

In the remainder of this section we state three simple lemmas. 

\begin{figure}
\centering
\includegraphics[clip=true,trim=6cm 19.2cm 6cm 2.5cm]{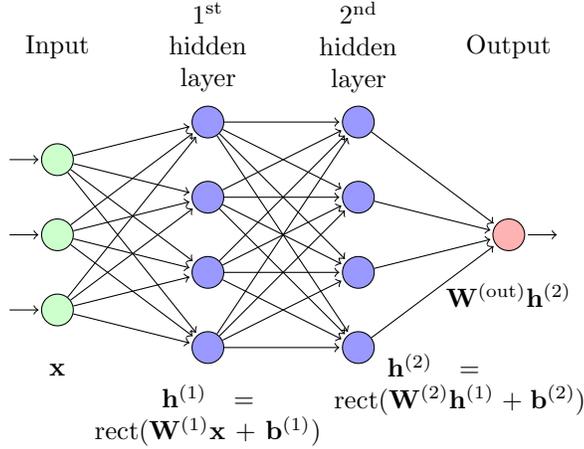}
\caption{Illustration of a rectifier feedforward network with two hidden layers. }
\label{figure:nn}
\end{figure}

The next lemma states that a piecewise linear function $f = ( f_i)_{i\in[k]}$
has as many regions of linearity as there are distinct intersections of regions
of linearity of the coordinates $f_i$. 
    
\begin{lemma}\label{lemma:nrregions} 
    Consider a width $k$ layer of rectifier units. 
    Let $R^i=\{R^i_1,\ldots, R^i_{N_i} \}$ be the regions of linearity of the
    function $f_i\colon \R^{n_0}\to\R$ computed by the $i$-th unit, for all $i\in [k]$. 
    Then the regions of linearity of the function $f = (f_i)_{i\in[k]}\colon \R^{n_0}
    \to \R^k$ computed by the rectifier layer are the elements of the set 
    $\{ R_{j_1,\ldots,j_k} = R^1_{j_1}\cap \cdots \cap
    R^k_{j_k}\}_{ (j_1,\ldots, j_k) \in [N_1]\times\cdots \times [N_k]}$.  
\end{lemma}

\begin{proof} 
    A function $f=(f_1,\ldots, f_k) \colon \R^n \to\R^k$ is linear iff all its
    coordinates $f_1, \ldots, f_k$ are.  
\end{proof}

In regard to the number of regions of linearity of the functions represented by
rectifier networks, the number of output dimensions, i.e., the number of linear
output units, is irrelevant. This is the statement of the next lemma. 

\begin{lemma}\label{lemma:partitions1} \label{lemma:multiple_outputs}
The number of (linear) output units of a rectifier feedforward network does
not affect the maximal number of regions of linearity that it can realize.
\end{lemma}

\begin{proof}
Let $f\colon \R^{n_0}\to\R^{k}$ be the map of inputs to activations in the last
hidden layer of a deep feedforward rectifier model.  Let $h = g \circ f$ be the
map of inputs to activations of the output units, given by composition of $f$
with the linear output layer,  $h (\mathbf{x}) = \mathbf{W}^{\text{(out)}}
f(\mathbf{x}) + \mathbf{b}^{\text{(out)}}$.  If the row span of
$\mathbf{W}^{\text{(out)}}$ is not orthogonal to any difference of gradients of
neighbouring regions of linearity of $f$, then $g$ captures all discontinuities
of $\nabla f$.  In this case both functions $f$ and $h$ have the same number of
regions of linearity.  

If the number of regions of $f$ is finite, then the number of differences of
gradients is finite and there is a vector outside the union of their orthogonal
spaces.  Hence a matrix with a single row (a single output unit) suffices to
capture all transitions between different regions of linearity of $f$. 
\end{proof}

\begin{lemma}
    \label{lemma:decomposition}
A layer of $n$ rectifier units with $n_0$ inputs can compute any function that
can be computed by the composition of a linear layer with $n_0$ inputs and
$n_0'$ outputs and a rectifier layer with $n_0'$ inputs and $n_1$ outputs, for
any $n_0,n_0',n_1\in\mathbb{N}$. 
\end{lemma}
\begin{proof}
A rectifier layer computes functions of the form $\vx \mapsto
\operatorname{rect}(\mW \vx + \vb)$, with $\mW\in\R^{n_1\times n_0}$ and
$\vb\in\R^{n_1}$.  The argument $\mW \vx + \vb$ is an affine function of $\vx$.
The claim follows from the fact that any composition of affine functions is an
affine function. 
\end{proof}

\section{One hidden layer}
\label{section:onelayer}

Let us look at the number of response regions of a single hidden layer MLP
with $n_0$ input units and $n$ hidden units. 
We first formulate the rectifier unit as follows:

\begin{equation}
    \label{eq:indicator_rect}
    \text{rect}(s) = \mathbb{I}(s) \cdot s,
\end{equation}
where $\mathbb{I}$ is the indicator function defined as:

\begin{equation}
    \label{eq:indicator}
    \mathbb{I}(s) = \begin{cases} 1  , &\text{if } s>0 \\
            0  , & \text{otherwise} 
        \end{cases}  .
\end{equation}

We can now write the single hidden layer MLP with $n_y$ outputs as the function
$f\colon \R^{n_0}\to\R^{n_y}$; 

\begin{equation}
    \label{eq:rewrite_MLP}
    \begin{array}{ll} 
    f(\mathbf{x}) &= 
    \mathbf{W}^{\text{(out)}}
    \text{diag}\left(
        \left[
            \begin{array}{c}
                \mathbb{I}(\mathbf{W}^{(1)}_{1:}\mathbf{x} + \mathbf{b}^{(1)}_1) \\
                \vdots \\
            \mathbb{I}(\mathbf{W}^{(1)}_{n_1:}\mathbf{x} + \mathbf{b}^{(1)}_{n_1}) \\
            \end{array}
        \right]\right)
        \left( \mathbf{W}^{(1)} \mathbf{x}  +  \mathbf{b}^{(1)} \right) 
        + \mathbf{b}^{\text{(out)}}
    \end{array} . 
\end{equation}

From this formulation it is clear that each unit $i$ in the hidden layer has
two operational modes. One is when the unit takes value $0$ and one when it takes a non-zero value.
The boundary between these two operational modes is given by the hyperplane
$H_i$ consisting of all inputs $\vx\in\R^{n_0}$ with
$\mathbf{W}^{(1)}_{i,:}\mathbf{x} + \mathbf{b}^{(1)}_i = 0$.  Below this
hyperplane, the activation of the unit is constant equal to zero, and above, it
is linear with gradient equal to $\mW_{i,:}^{(1)}$.  It follows that the number
of regions of linearity of a single layer MLP is equal to the number of regions
formed by the set of hyperplanes $\{ H_i\}_{i\in[n_1]}$. 

A finite set of hyperplanes in a common $n_0$-dimensional Euclidian space is called an $n_0$-dimensional {\em hyperplane arrangement}. 
A {\em region} of an arrangement $\Acal = \{H_i\subset \R^{n_0}\}_{i\in[n]}$ is a connected
component of the complement of the union of the hyperplanes, i.e., a connected component of $\R^{n_0}\setminus (\cup_{i\in[n]} H_i)$. 
To make this clearer, consider an arrangement $\Acal$ consisting of hyperplanes $H_i=\{\vx\in\R^{n_0}\colon
\mW_{i,:} \vx + \vb_i=0 \}$ for all $i\in[n]$, for some $\mW\in\R^{n\times n_0}$ and some $\vb\in\R^{n}$. 
A region of $\Acal$ is a set of points of the form $R
=\{\vx\in\R^{n_0} \colon \operatorname{sgn} (\mW \vx +  \vb ) = \vs \}$ for some sign vector $\vs\in\{-,+\}^n$. 

A region of an arrangement is {\em relatively bounded} if its intersection with the space spanned
by the normals of the hyperplanes is bounded.  We denote by $r(\Acal)$ the
number of regions and by $b(\Acal)$ the number of relatively bounded regions of
an arrangement $\Acal$. 
The {\em essentialization} of an arrangement $\mathcal{A}= \{H_i \}_i$ is the arrangement consisting of the hyperplanes 
$H_i \cap \mathcal{N}$ for all $i$, defined in the span $\mathcal{N}$ of the normals
of the hyperplanes $H_i$. For example, the essentialization of an arrangement of two
non-parallel planes in $\R^3$ is an arrangement of two lines in a plane.

\begin{problem} 
    \label{problem1} 
    How many regions are generated by an arrangement of $n$ hyperplanes in $\R^{n_0}$?  
\end{problem}

The general answer to Problem~\ref{problem1} is given by Zaslavsky's
theorem~\citep[][Theorem~A]{zaslavsky1975facing}, which is one of the central
results from the theory of hyperplane arrangements. 

We will only need the special case of hyperplanes in \textit{general position},
which realize the maximal possible number of regions.  Formally, an
$n$-dimensional arrangement $\Acal$ is in general position if for any subset
$\{H_1,\ldots H_p \}\subseteq \Acal$ the following holds.  (1)~~\!If $p\leq n$,
then $\dim(H_1\cap\cdots\cap H_p) = n-p$. (2)~~\!If $p > n$, then
$H_1\cap\cdots\cap H_p=\emptyset$.  An arrangement is in general position if
the weights $\mW$, $\vb$ defining its hyperplanes are generic. This means that any arrangement
can be perturbed by an arbitrarily small perturbation in such a way that the
resulting arrangement is in general position. 

For arrangements in general position, Zaslavsky's theorem can be stated in the
following way~\cite[see][Proposition~2.4]{Stanley04}. 

\begin{proposition}\label{proposition:genericarrangement}
    Let $\Acal$ be an arrangement of $m$ hyperplanes in general position in
    $\R^{n_0}$.  Then
    \begin{eqnarray*} 
        r(\Acal) & = & \sum_{s=0}^{n_0}{m\choose s}\\ 
        b(\Acal) & = & {m-1\choose n_0}.
\end{eqnarray*} \end{proposition}

In particular, the number of regions of a $2$-dimensional arrangement $\Acal_m$
of $m$ lines in general position is equal to 
\begin{equation}
    \label{eq:region_2d}
r(\Acal_m) = 
{m \choose 2} + m + 1 . 
\end{equation}

    \begin{figure}[t]
    \bigskip
    \centering 
    \includegraphics[scale=.5]{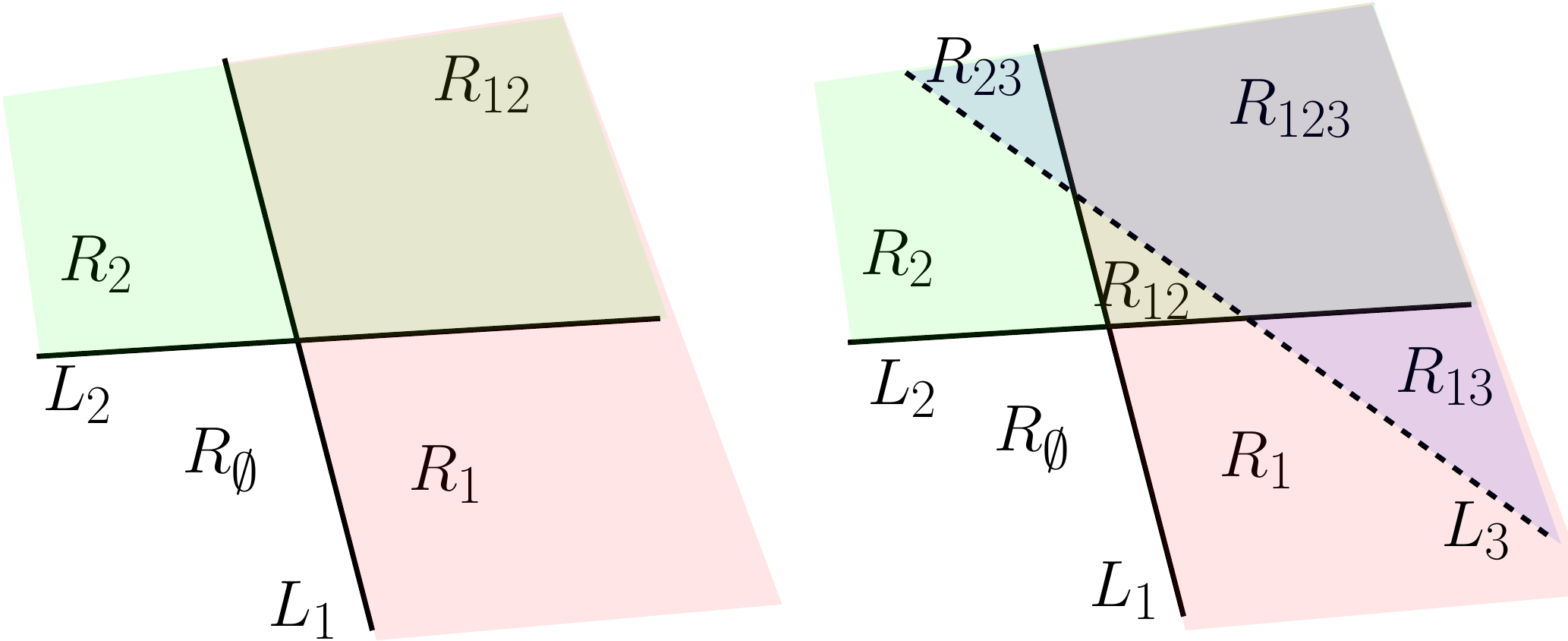} 
    \caption{Induction step of the hyperplane sweep method for counting the regions of line arrangements in the plane. 
    }
    \label{fig:sweep_ind_case}
\end{figure}

For the purpose of illustration, we sketch a proof of eq.~\eqref{eq:region_2d}
using the \textit{sweep hyperplane} method.  We proceed by induction over the
number of lines $m$.

\textit{Base case $m = 0$. } It is obvious that in this case there is a
single region, corresponding to the entire plane. Therefore, $r(\Acal_0) =
1$. 

\textit{Induction step.} Assume that for $m$ lines the number of regions is
$r(\Acal_{m}) = {m \choose 2} + m+1$, and add a new line $L_{m+1}$ to the
arrangement.  Since we assumed the lines are in general position, $L_{m+1}$
intersects each of the existing lines $L_k$ at a different point.
Fig.~\ref{fig:sweep_ind_case} depicts the situation for $m=2$.

The $m$ intersection points split the line $L_{m+1}$ into $m+1$ segments.
Each of these segments cuts a region of $\mathcal{A}_m$ in two pieces.
Therefore, by adding the line $L_{m+1}$ we get $m+1$ new regions.  In
Fig.~\ref{fig:sweep_ind_case} the two intersection points result in three
segments that split each of the regions $R_{1}, R_{01}, R_{0}$ in two.
Hence
\begin{align*}
    r(\Acal_{m+1})& = r(\Acal_{m}) + m + 1= \frac{m(m-1)}{2} + m + 1 + m+1 
                  = \frac{m(m+1)}{2} + (m+1)+1 \\
                  & = {m+1 \choose 2} + (m+1) + 1 . 
    \end{align*}

For the number of response regions of MLPs with one single hidden layer we obtain the following. 

\begin{proposition}\label{proposition:onelayermodel}
    The regions of linearity of a function in the model $\Fcal_{(n_0,n_1,1)}$
    with $n_0$ inputs and $n_1$ hidden units are given by the regions of an
    arrangement of $n_1$ hyperplanes in $n_0$-dimensional space.  The maximal
    number of regions of such an arrangement is
    $\Rcal(n_0,n_1,n_y)=\sum_{j=0}^{n_0}{ n_1 \choose j} $.  
\end{proposition} 
    
\begin{proof} 
    This is a consequence of Lemma~\ref{lemma:nrregions}.  The
    maximal number of regions is produced by an $n_0$-dimensional arrangement
    of $n_1$ hyperplanes in general position, which is given in
    Proposition~\ref{proposition:genericarrangement}.  
\end{proof}

\section{Multiple hidden layers} 
\label{section:klayers}

In order to show that a $k$ hidden layer model can be more expressive than a
single hidden layer one with the same number of hidden units, 
we will need the next three propositions. 

\begin{proposition}
    \label{proposition:arrangement_in_sphere}
Any arrangement can be scaled down and shifted such that all regions of the
arrangement intersect the unit ball. 
\end{proposition}

\begin{figure}[ht]
    \centering 
    \includegraphics[scale=.4]{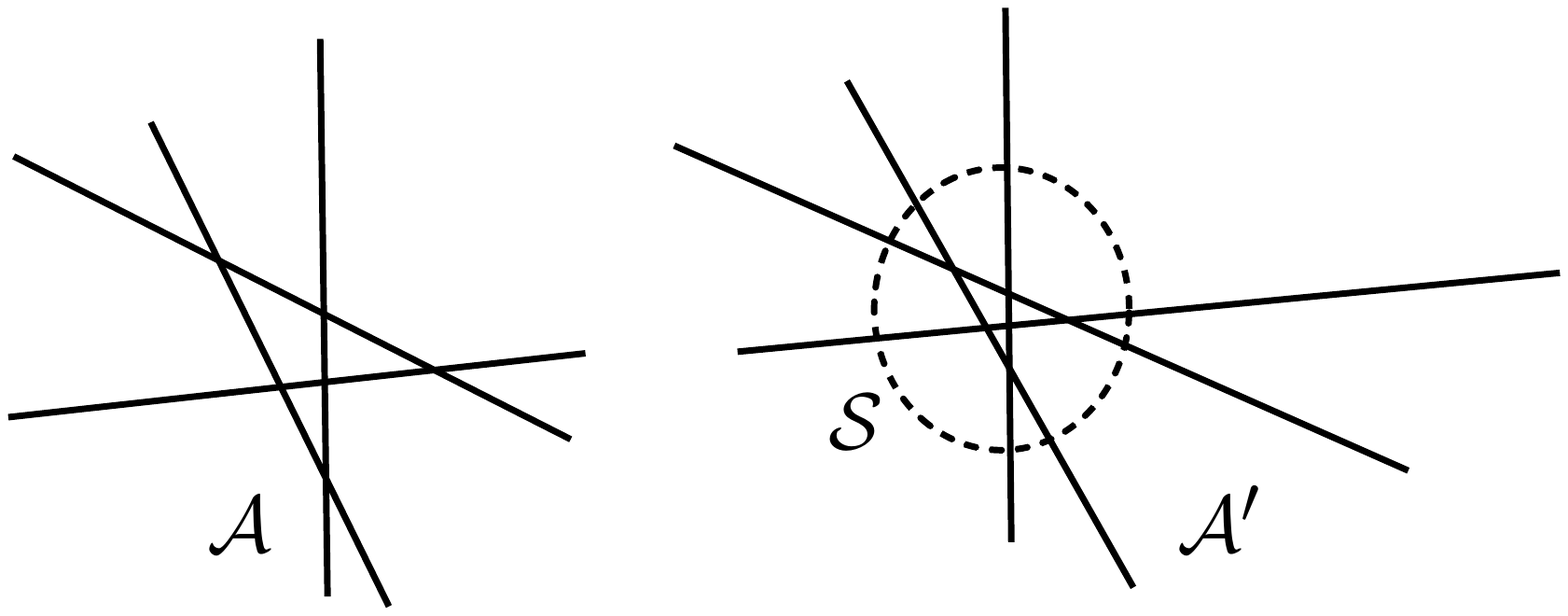} 
    \caption{
An arrangement $\mathcal{A}$ and a scaled-shifted version $\mathcal{A}'$ whose regions intersect the ball $\mathcal{S}$.  
} 
    \label{fig:scaling}
\end{figure}
\begin{proof}
    Let $\Acal$ be an arrangement and let $\mathcal{S}$ be a ball of radius $r$
    and center $\mathbf{c}$. 
    Let $d$ be the supremum of the distance from the origin to a point in a
    bounded region of the essentialization of the arrangement $\Acal$.
    Consider the map $\phi:\R^{n_0}\to\R^{n_0}$ defined by $\phi(\mathbf{x}) =
    \frac{r}{2d} \cdot \mathbf{x} + \mathbf{c}$.  Then $\Acal' = \phi(\Acal)$
    is an arrangement satisfying the claim. 
    It is easy to see that any point with norm bounded by $d$ is mapped to a
    point inside the ball $S$. 
\end{proof}

The proposition is illustrated in Fig.~\ref{fig:scaling}. 

We need some additional notations in order to formulate the next proposition.
Given a hyperplane $H = \{\vx \colon \vw^\top \vx +  b =0\}$, we consider the
region $H^- = \{ \vx \colon \vw^\top \vx + b <0\}$, and the region $H^+ = \{
\vx \colon \vw^\top \vx +  b \geq 0\}$.  If we think about the corresponding
rectifier unit, then $H^+$ is the region where the unit is active and $H^-$ is
the region where the unit is dead.  

Let $R$ be a region delimited by the hyperplanes $\{H_1, \ldots, H_n\}$. We
denote by $R^+\subseteq \{1,\ldots, n\}$ the set of all hyperplane-indices $j$
with $R \subset H^+_j$.  In other words, $R^+$ is the list of hidden units that
are active (non-zero) in the input-space region $R$. 

The following proposition describes the combinatorics of $2$-dimensional
arrangements in general position.  More precisely, the proposition describes
the combinatorics of $n$-dimensional arrangements with $2$-dimensional
essentialization in general position.  Recall that the essentialization of an
arrangement is the arrangement that it defines in the subspace spanned by the
normals of its hyperplanes. 

The proposition guarantees the existence of input weights and bias for a
rectifier layer such that for any list of consecutive units, there is a region
of inputs for which exactly the units from that list are active. 

\begin{proposition}
    \label{proposition:special_arrangement}
    For any $n_0,n\in\mathbb{N}$, $n\geq 2$, there exists an $n_0$-dimensional
    arrangement $\mathcal{A}$ of $n$ hyperplanes such that for any pair
    $a,b\in\{1,\ldots, n\}$ with $a < b$, there is a region $R$ of
    $\mathcal{A}$ with $R^+ = \{a, a+1, \ldots, b\}$. 
\end{proposition}

We show that the hyperplanes of a $2$-dimensional arrangement in general
position can be indexed in such a way that the claim of the proposition holds.
For higher dimensional arrangements the statement follows trivially, applying
the $2$-dimensional statement to the intersection of the arrangement with a
$2$-subspace. 

\begin{proof}[Proof of Proposition~\ref{proposition:special_arrangement}] 
 Consider first the case $n_0=2$.  We define the first line $L_1$ of the
 arrangement to be the x-axis of the standard coordinate system.  To define the
 second line $L_2$, we consider a circle $\mathcal{S}_1$ of radius $r\in\R_+$
 centered at the origin.  We define $L_2$ to be the tangent of $\mathcal{S}_1$
 at an angle $\alpha_1$ to the y-axis, where $0<\alpha_1 < \frac{\pi}{2}$.  The
 top left panel of Fig.~\ref{fig:twolines} depicts the situation.  In the
 figure, $R_\emptyset$ corresponds to inputs for which no rectifier unit is
 active, $R_1$ corresponds to inputs where the first unit is active, $R_2$ to
 inputs where the second unit is active, and $R_{12}$ to inputs where both
 units are active.  This arrangement has the claimed properties. 

 Now assume that there is an arrangement of $n$ lines with the claimed
 properties.  To add an \mbox{$(n+1)$-th} line, we first consider the maximal
 distance $d_{\text{max}}$ from the origin to the intersection of two lines
 $L_i \cap L_j$ with $1\leq i < j < n$.  We also consider the
 radius-$(d_{\text{max}}+r)$ circle $\mathcal{S}_{n}$ centered at the origin.
 The circle $\mathcal{S}_{n}$ contains all intersection of any of the first $n$
 lines.  We now choose an angle $\alpha_n$ with $0 < \alpha_{n} < \alpha_{n-1}$
 and define $L_{n+1}$ as the tangent of $\mathcal{S}_{n}$ that forms an angle
 $\alpha_{n}$ with the y-axis.  Fig.~\ref{fig:morelines} depicts adding the
 third and fourth line to the arrangement. 

 \begin{figure}[ht]
    \centering 
    \smallskip
    \begin{minipage}{.4\textwidth}
    \includegraphics[width=\textwidth]{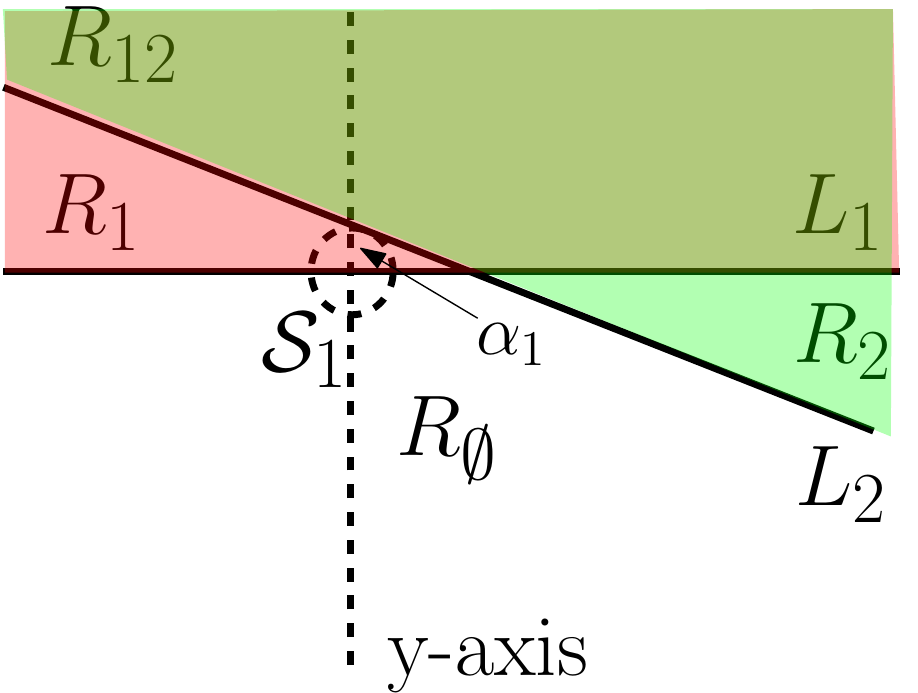} 
    \includegraphics[width=\textwidth]{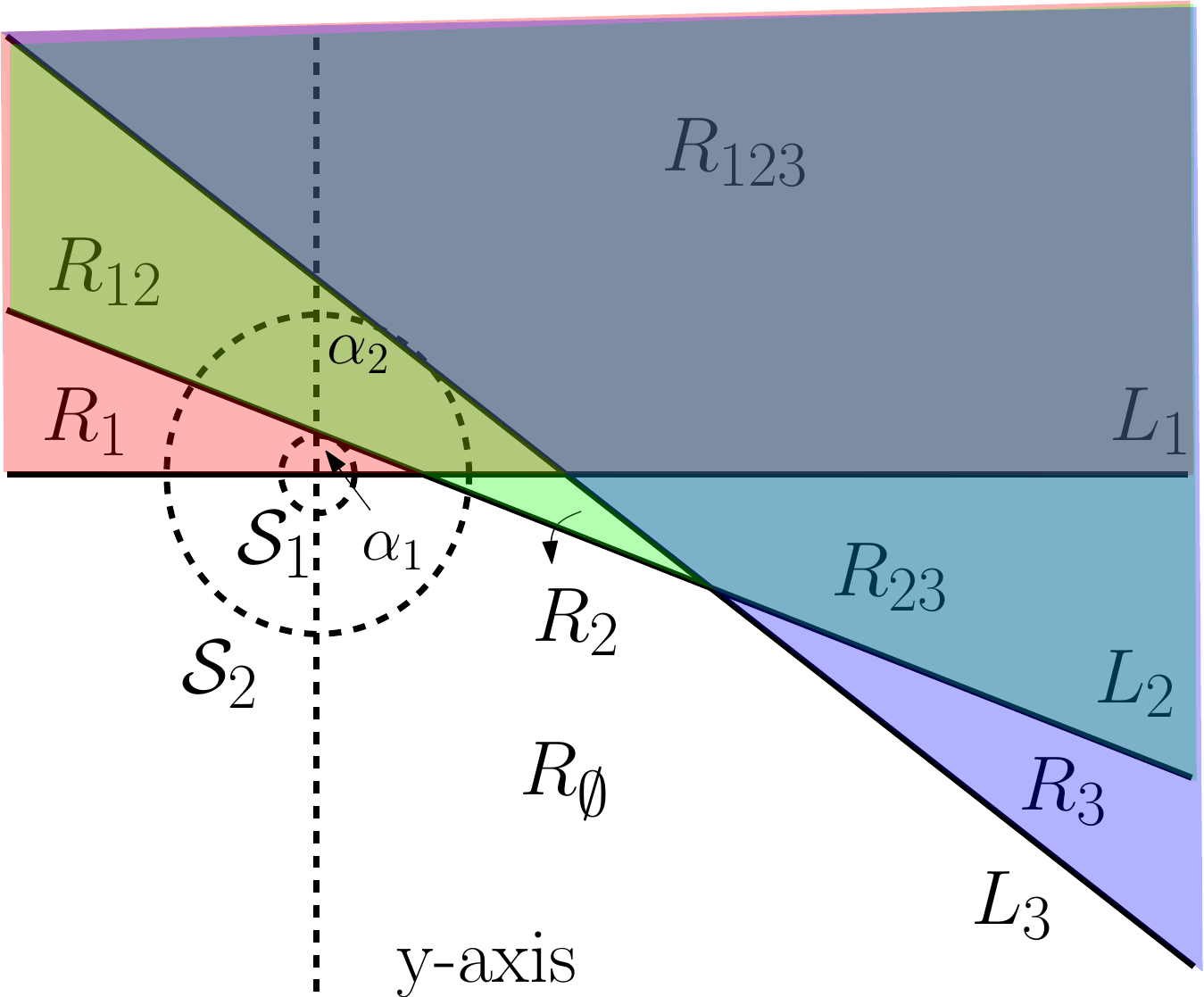} 
    \end{minipage}
    \begin{minipage}{.53\textwidth}
    \includegraphics[width=\textwidth]{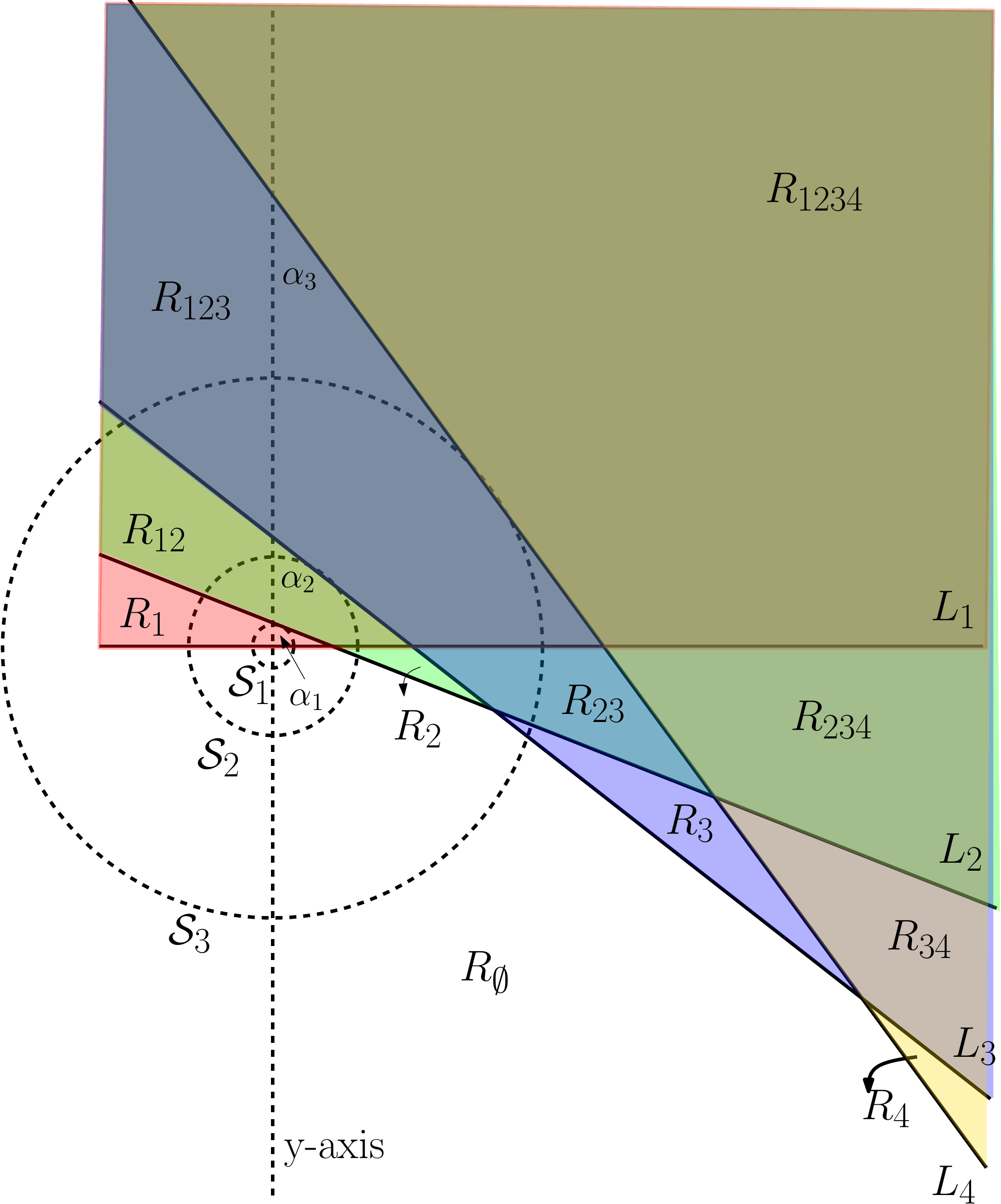} 
    \end{minipage}
    \caption{Illustration of the hyperplane arrangement discussed in
        Proposition~\ref{proposition:special_arrangement}, in the
        $2$-dimensional case.  On the left we have arrangements of two and
        three lines, and on the right an arrangement of four lines.
    \label{fig:morelines} \label{fig:twolines}
    } 
 \end{figure}

After adding line $L_{n+1}$, we have that the arrangement 
\begin{enumerate}
    \item is in general position. 
    \item has regions $R_1', \ldots, R_{n+1}'$ with $R_i'^+ = \{i, i+1, \ldots, n+1\}$ for all $i\in[n+1]$. 
\end{enumerate}

The regions of the arrangement are stable under perturbation of the angles and
radii used to define the lines. 
Any slight perturbation of these parameters preserves the list of regions. 
Therefore, the arrangement is in general position. 

The second property comes from the order in which $L_{n+1}$ intersects all
previous lines.  $L_{n+1}$ intersects the lines in the order in which they were
added to the arrangement: $L_1, L_2, \ldots,  L_n$.  The intersection of
$L_{n+1}$ and $L_i$, $B_{i{n+1}} = L_{n+1} \cap L_i$, is above the lines
$L_{i+1}, L_{i+2}, \ldots, L_{n}$, and hence the segment
$B_{{i-1}{n+1}}B_{i{n+1}}$  between the intersection with $L_{i-1}$ and with
$L_i$, has to cut the region in which only units $i$ to $n$ are active. 

The intersection order is ensured by the choice of angles $\alpha_i$ and the
fact that the lines are tangent to the circles $\mathcal{S}_i$.  For any $i<j$
and $B_{ij} = L_i \cap L_j$ let $T_{ij}$ be the line parallel to the y-axis
passing through $B_{ij}$.  Each line  $T_{ij}$ divides the space in two.  Let
$H_{ij}$ be the half-space to the right of $T_{ij}$.  Within any half-space
$H_{ij}$, the intersection $H_{ij} \cap L_i$ is above $H_{ij} \cap L_j$,
because the angle $\alpha_{i-1}$ of $L_i$ with the y-axis is larger than
$\alpha_{j-1}$ (this means $L_j$ has a stepper decrease).  Since $L_{n+1}$ is
tangent to the circle that contains all points $B_{ij}$, the line $L_{n+1}$
will intersect lines $L_i$ and $L_j$ in $H_{ij}$, and therefore it has to
intersect $L_i$ first.

For $n_0>2$ we can consider an arrangement that is essentially $2$-dimensional
and has the properties of the arrangement described above.  To do this, we
construct a $2$-dimensional arrangement in a $2$-subspace of $\R^{n_0}$ and
then extend each of the lines $L_i$ of the arrangement to a hyperplane $H_i$
that crosses $L_i$ orthogonally.  The resulting arrangement satisfies all
claims of the proposition.
\qedhere
\end{proof}

The next proposition guarantees the existence of a collection of affine maps with shared bias, 
which map a collection of regions to a common output. 

\begin{proposition}\label{proposition:mini}
Consider two integers $n_0$ and $p$.  Let $\mathcal{S}$ denote the
$n_0$-dimensional unit ball and let $R_1,\ldots,R_p\subseteq\R^{n_0}$ be some
regions with non-empty interiors.  Then there is a choice of weights
$\mathbf{c}\in\R^{n_0}$ and $ \mathbf{U}_1, \ldots,\mU_{p}\in\R^{n_0\times
n_0}$ for which $g_i( R_i) \supseteq \mathcal{S} $ for all $i \in [p]$,
where $g_i\colon \R^{n_0}\to \R^{n_0};\, \vy \mapsto \mU_i \vy + \vc$. 
\end{proposition}
\begin{proof}
 To see this, consider the following construction.  For each region $R_i$
 consider a ball $\mathcal{S}_i\subseteq R_i$ of radius $r_i\in\R_+$ and center
 $\mathbf{s}_i = (\vs_{i 1},\ldots, \vs_{i n_0})\in \R^{n_0}$.  
 For each $j=1,\ldots, n_0$, consider $p$ positive numbers $u_{1j},
 \ldots , u_{p j}$ such that $u_{ij} {\mathbf{s}_i}_j = u_{kj}
 {\mathbf{s}_k}_j$ for all $1 \leq k < i \leq p$.  This can be done fixing
 $u_{1j}$ equal to $1$ and solving the equation for all other numbers.  Let
 $\eta\in\R$ be such that $r_i \eta u_{ij}> 1$ for any $j$ and $i$.  Scaling
 each region $R_i$ by $\mathbf{U}_i = \operatorname{diag}( \eta u_{i0}, \ldots,
 \eta u_{i n_0})$ transforms the center of $\mathcal{S}_i$ to the same point
 for all $i$.  By the choice of $\eta$, the minor radius of all transformed
 balls is larger than $1$. 
 
 We can now set $\mathbf{c}$ to be minus the common center of the scaled balls,
 to obtain the map:
  \begin{equation*}
      g_i(\mathbf{x}) = \text{diag}\left(\eta u_{i1}, \ldots,\eta u_{i n_0} \right) \mathbf{x} -  
      \text{diag}\left(\eta u_{11}, \ldots, \eta u_{1 n_0}\right) \mathbf{s}_1, \quad \text{ for all $1\leq i\leq p$} . 
  \end{equation*}

These $g_i$ satisfy claimed property, namely that $g_i(R_i)$ contains the unit ball, for all $i$. 
\end{proof}

\begin{figure}
\centering
\includegraphics[clip=true,trim=6cm 17.5cm 5.5cm 4.5cm]{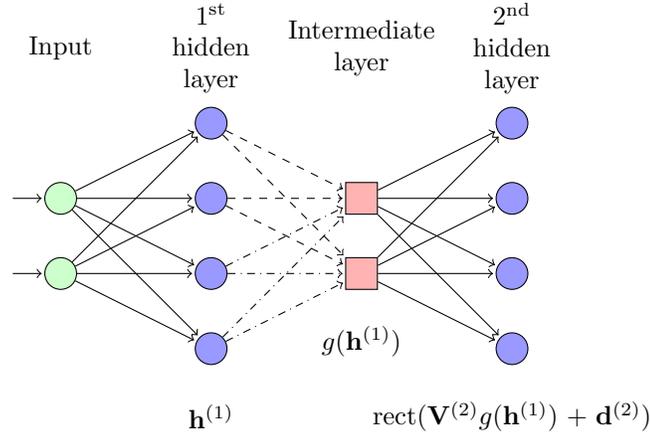}
\caption{Illustration of Example~\ref{example:klayermodel}. 
        The units represented by squares build an intermediary layer of linear
        units between the first and the second hidden layers.  The computation
        of such an intermediary linear layer can be absorbed in the second
        hidden layer of rectifier units (Lemma~\ref{lemma:decomposition}).  The
        connectivity map depicts the maps $g_1$ by dashed arrows and $g_2$ by
        dashed-dotted arrows. 
}
    \label{fig:depiction_interlayer}
\end{figure}

Before proceeding, we discuss an example illustrating how the previous
propositions and lemmas are put together to prove our main result below, in
Theorem~\ref{theorem:klayermodel}. 

\begin{example} 
\label{example:klayermodel}
Consider a rectifier MLP with $n_0=2$, such that the input space is $\R^2$, and
assume that the network has only two hidden layers, each consisting of $n=2n'$
units.  Each unit in the first hidden layer defines a hyperplane in $\R^2$,
namely the hyperplane that separates the inputs for which it is active, from
the inputs for which it is not active.  Hence the first hidden layer defines an
arrangement of $n$ hyperplanes in $\R^2$.  By
Proposition~\ref{proposition:special_arrangement}, this arrangement can be made
such that it delimits regions of inputs $R_1$, \ldots, $R_{n'}\subseteq\R^2$
with the following property.  For each input in any given one of these regions,
exactly one pair of units in the first hidden layer is active, and,
furthermore, the pairs of units that are active on different regions are
disjoint. 

By the definition of rectifier units, each hidden unit computes a linear
function within the half-space of inputs where it is active.  In turn, the
image of $R_i$ by the pair of units that is active in $R_i$ is a polyhedron in
$\R^2$.  For each region $R_i$, denote corresponding polyhedron by $S_i$.

Recall that a rectifier layer computes a map of the form $f \colon
\R^{n}\to\R^{m};\, \vx \mapsto \operatorname{rect}(\mW \vx +\vb)$.  Hence a
rectifier layer with $n$ inputs and $m$ outputs can compute any composition
$f'\circ g$ of an affine map $g\colon \R^n\to\R^{k}$ and a map $f'$ computed by
a rectifier layer with $k$ inputs and $m$ outputs
(Lemma~\ref{lemma:decomposition}). 

Consider the map computed by the rectifier units in the second hidden layer,
i.e., the map that takes activations from the first hidden layer and outputs
activations from the second hidden layer.  We think of this map as a
composition $f'\circ g$ of an affine map $g\colon\R^n \to \R^2$ and a map $f'$
computed by a rectifier layer with $2$ inputs.  The map $g$ can be interpreted
as an intermediary layer consisting of two linear units, as illustrated in
Fig.~\ref{fig:depiction_interlayer}. 

Within each input region $R_i$, only two units in the first hidden layer are
active.  Therefore, for each input region $R_i$, the output of the intermediary
layer is an affine transformation of $S_i$.  Furthermore, the weights of the
intermediary layer can be chosen in such a way that the image of each $R_i$
contains the unit ball.

Now, $f'$ is the map computed by a rectifier layer with $2$ inputs and $n$
outputs.  It is possible to define this map in such a way that it has $\Rcal$
regions of linearity within the unit ball, where $\Rcal$ is the number of
regions of a $2$-dimensional arrangement of $n$ hyperplanes in general
position. 

We see that the entire network computes a function which has $\Rcal$ regions of
linearity within each one of the input regions $R_1,\ldots, R_{n'}$.  Each
input region $R_i$ is mapped by the concatenation of first and intermediate
(notional) layer to a subset of $\R^2$ which contains the unit ball.  Then, the
second layer computes a function which partitions the unit ball into many
pieces.  The partition computed by the second layer gets replicated in each of
the input regions $R_i$, resulting in a subdivision of the input space in
exponentially many pieces (exponential in the number of network layers).
\end{example}

Now we are ready to state our main result on the number of response regions of rectifier deep feedforward networks: 

 \begin{theorem}\label{theorem:klayermodel}
     A model with $n_0$ inputs and $k$ hidden layers of widths $n_1, n_2, \ldots, n_k$
     can divide the input space in $\left(\prod_{i=1}^{k-1}\left\lfloor\frac{n_i}{n_0}\right\rfloor\right)
      \sum_{i=0}^{n_0} {n_k \choose i} $ or possibly more regions.
 \end{theorem}

\begin{figure} 
    \bigskip
    \centering 
    \includegraphics[scale=.4]{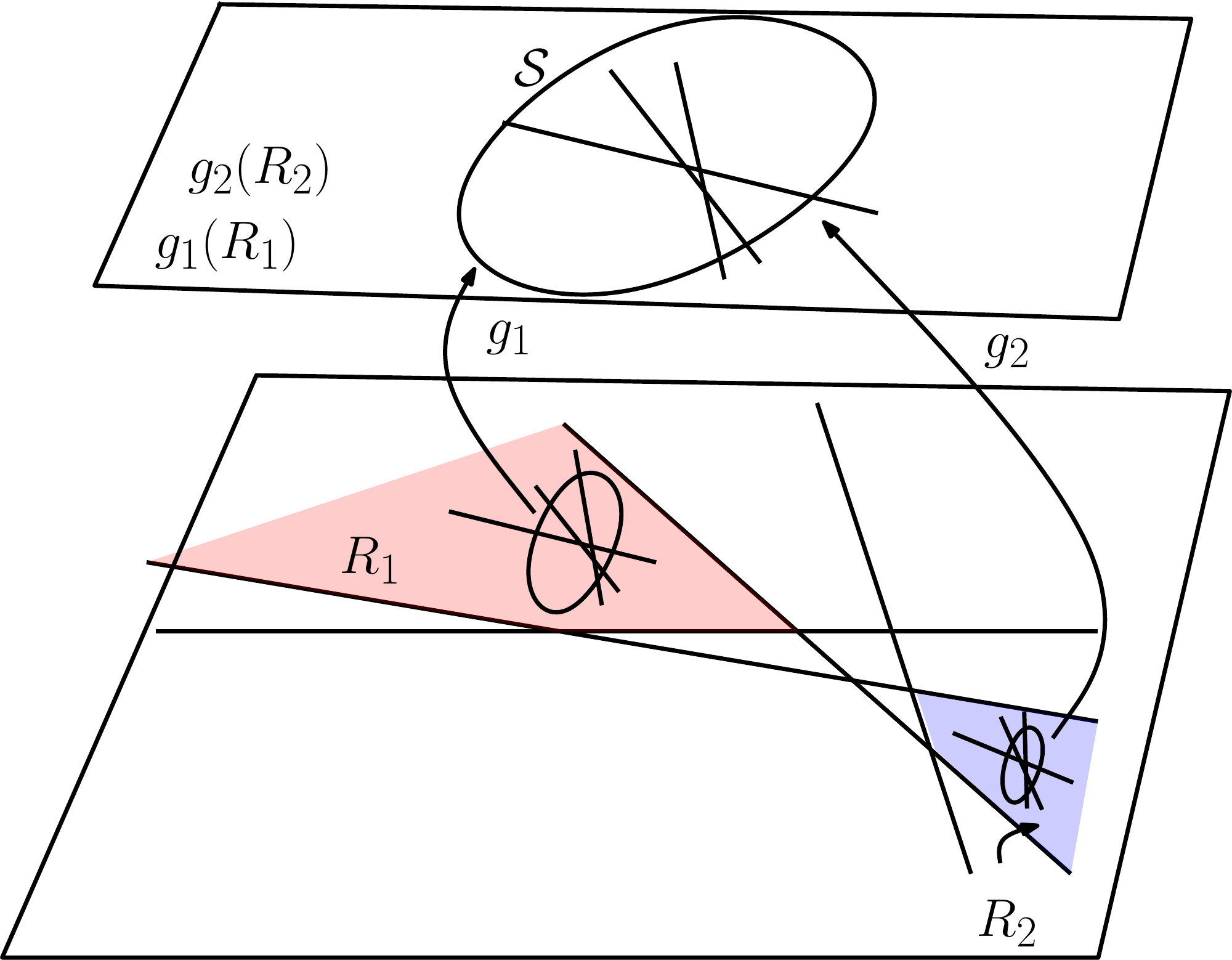} 
    \caption{Constructing $\left\lfloor \frac{n_1}{n_0}\right\rfloor
        \sum_{k=0}^{n_0} {n_2 \choose k}$ response regions in a model with two layers.
    }
    \label{fig:depiction_2layers}
\end{figure}

 \begin{proof}[Proof of Theorem~\ref{theorem:klayermodel}]
     Let the first hidden layer define an arrangement like the one from
     Proposition~\ref{proposition:special_arrangement}.  Then there are $p =
     \left\lfloor \frac{n_1}{n_0}\right\rfloor$ input-space regions
     $R_i\subseteq \R^{n_0}$, $i\in[p]$ with the following property.  For each
     input vector from the region $R_i$, exactly $n_0$ units from the first
     hidden layer are active.  We denote this set of units by $I_i$.
     Furthermore, by Proposition~\ref{proposition:special_arrangement}, for
     inputs in distinct regions $R_i$, the corresponding set of active units is
     disjoint; that is, $I_i\cap I_j=\emptyset$ for all $ i, j\in [p]$, $i\neq
     j$.  
     
     To be more specific, for an input vectors from $R_1$, exactly the first
     $n_0$ units of the first hidden layer are active, that is, for these input
     vectors the value of $\vh^{(1)}_j$ is non-zero if and only if $j\in I_1=
     \{1,\ldots, n_0\}$.  For input vectors from $R_2$, only the next $n_0$
     units of the first hidden layer are active, that is, the units with index
     in $I_2=\{n_0+1,\ldots, 2 n_0  \}$, and so on. 

Now we consider a `fictitious' intermediary layer consisting of $n_0$ linear
units between the first and second hidden layers.  As this intermediary layer
computes an affine function, it can be absorbed into the second hidden layer
(see Lemma~\ref{lemma:decomposition}).  We use it only for making the next
arguments clearer. 

The map taking activations from the first hidden layer to activations from the
second hidden layer is $\operatorname{rect}(\mathbf{W}^{(2)} \vx +\vb^{(2)})$,
where $\mathbf{W}^{(2)} \in \R^{n_2 \times n_1}, \mathbf{b}^{(2)} \in
\R^{n_2}$. 

We can write the input and bias weight matrices as $\mathbf{W}^{(2)} =
\mathbf{U}^{(2)} \mathbf{V}^{(2)}$ and $\mathbf{b}^{(2)} = \mathbf{d}^{(2)} +
\mathbf{V}^{(2)} \mathbf{c}^{(2)}$, where $\mathbf{U}^{(2)} \in \R^{n_0 \times
n_1}$, $\mathbf{c}\in \R^{n_0}$, and $\mathbf{V}^{(2)} \in \R^{n_2 \times
n_0}$, $\mathbf{d}\in \R^{n_2}$. 

The weights $\mathbf{U}^{(2)}$ and $\mathbf{c}^{(2)}$ describe the affine
function computed by the intermediary layer, $\vx \mapsto \mathbf{U}^{(2)}\vx +
\mathbf{c}$.  The weights $\mathbf{V}^{(2)}$ and $\mathbf{d}^{(2)}$ are the
input and bias weights of the rectifier layer following the intermediary layer. 

We now consider the sub-matrix $\mU^{(2)}_i$ of $\mU^{(2)}$ consisting of the
columns of $\mU^{(2)}$ with indices $I_i$, for all $i\in[p]$.  Then $\mU^{(2)}
= \left[ \mU^{(2)}_1 | \cdots | \mU^{(2)}_{p} |  \tilde{\mU}^{(2)}  \right]$,
where $\tilde{\mU}^{(2)}$ is the sub-matrix of ${\mU}^{(2)}$ consisting of its
last $n_1-p n_0$ columns. In the sequel we set all entries of
$\tilde{\mU}^{(2)}$ equal to zero. 

The map $g:\R^{n_1}\to \R^{n_0};\, g(\mathbf{x}) = \mathbf{U}^{(2)} \mathbf{x}
+ \mathbf{c}^{(2)}$ is thus written as the sum $g=\sum_{i\in[p]} g_i +
\vc^{(2)}$, where $g_i: \R^{n_0} \to \R^{n_0}; \,  g_i(\mathbf{x}) =
\mathbf{U}^{(2)}_i \mathbf{x}$, for all $i\in[p]$. 

Let $S_i$ be the image of the input-space region $R_i$ by the first hidden
layer.  By Proposition~\ref{proposition:mini}, there is a choice of the weights
$\mathbf{U}^{(2)}_i$ and bias $\mathbf{c}^{(2)}$ such that the image of $S_i$
by $\mathbf{x} \to \mU^{(2)}_i(\mathbf{x}) + \mathbf{c}^{(2)}$ contains the
$n_0$-dimensional unit ball.  Now, for all inputs vectors from $R_i$, only the
units $I_i$ of the first hidden layer are active.  Therefore, $g|_{R_i} =
g_i|_{R_i} + \vc^{(2)}$.  This implies that the image $g(R_i)$ of the
input-space region $R_i$ by the intermediary layer contains the unit ball, for
all $i\in[p]$. 

We can now choose $\mathbf{V}^{(2)}$ and $\mathbf{d}^{(2)}$ in such a way that
the rectifier function $\R^{n_0}\to \R^{n_2};\, \vy \mapsto
\operatorname{rect}( \mV^{(2)} \vy + \mathbf{d}^{(2)} )$ defines an arrangement
$\Acal$ of $n_2$ hyperplanes with the property that each region of $\Acal$
intersects the unit ball at an open neighborhood. 

In consequence, the map from input-space to activations of the second hidden
layer has $r(\Acal)$ regions of linearity within each input-space region $R_i$.
Fig.~\ref{fig:depiction_2layers} illustrates the situation.  All inputs that
are mapped to the same activation of the first hidden layer, are treated as
equivalent on the subsequent layers.  In this sense, an arrangement $\Acal$
defined on the set of common outputs of $R_1, \ldots, R_p$ at the first hidden
layer, is `replicated' in each input region $R_1,\ldots, R_p$.

    \begin{figure} 
    \bigskip
    \centering 
    \includegraphics[scale=.4]{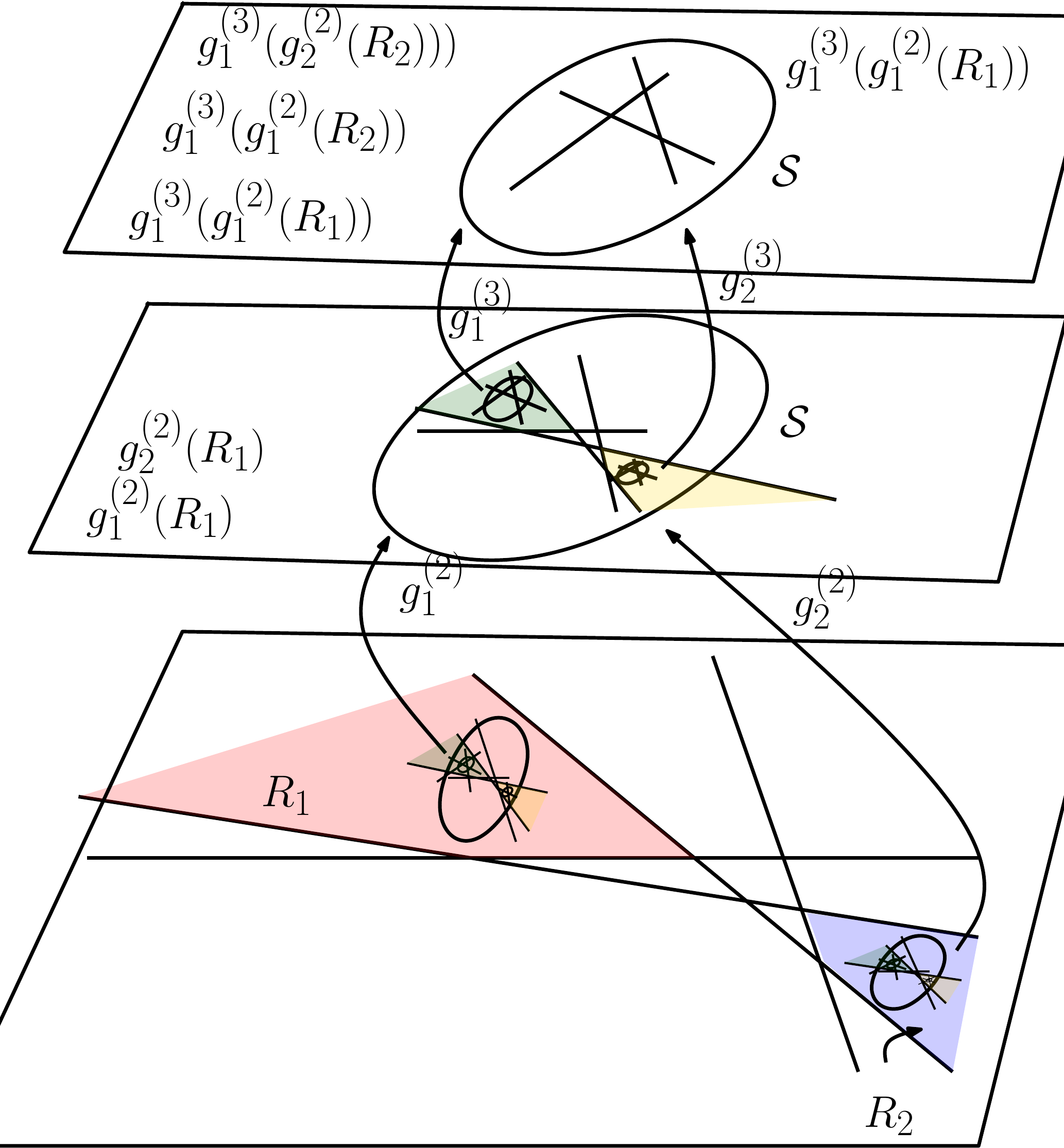} 
    \caption{Constructing $ \left\lfloor
        \frac{n_2}{n_0}\right\rfloor \left\lfloor
        \frac{n_1}{n_0}\right\rfloor \sum_{k=0}^{n_0} {n_3 \choose
        k} $ response regions in a model with three layers. 
    }
    \label{fig:depiction_nlayers}
\end{figure}

The subsequent layers of the network can be analyzed in a similar way as done
above for the first two layers. 
In particular, the weights $\mathbf{V}^{(2)}$ and $\mathbf{d}^{(2)}$ can be
chosen in such a way that they define an arrangement with the properties from
Proposition~\ref{proposition:special_arrangement}.  Then, the map taking
activations from the second hidden layer to activations from the third hidden
layer, can be analyzed by considering again a fictitious intermediary layer
between the second and third layers, and so forth, as done above. 

For the last hidden layer we choose the input weights $\mathbf{V}^{(k)}$ and
bias $\mathbf{d}^{(k)}$ defining an $n_0$-dimensional arrangement of $n_k$
hyperplanes in general position.  The map of inputs to activations of the last
hidden layer has thus
$\left(\prod_{i=1}^{k-1}\left\lfloor\frac{n_i}{n_0}\right\rfloor\right)
\sum_{i=0}^{n_0} {n_k \choose i} $ regions of linearity. 
This number is a lower bound on the maximal number of regions of linearity of
functions computable by the network. This completes the proof.  The intuition
of the construction is illustrated in Fig.~\ref{fig:depiction_nlayers}. 
 \end{proof}

In the Appendix~\ref{section:asymptotic} we derive an asymptotic expansion of
the bound given in Theorem~\ref{theorem:klayermodel}. 
 
 \section{A special class of deep models}
 \label{sec:opt1}
In this section we consider deep rectifier models with $n_0$ input units and
hidden layers of width $n = 2n_0$.  This restriction allows us to construct a
very efficient deep model in terms of number of response regions.  The analysis
that we provide in this section complements the results from the previous
section, showing that rectifier MLPs can compute functions with many response
regions, even when defined with relatively few hidden layers.

 \begin{example}\label{example:illustration}
 Let us assume we have a $2$-dimensional input, i.e., $n_0=2$, and a layer of
 $n=4$ rectifiers $f_1$, $f_2$, $f_3$, and $f_4$, followed by a linear
 projection.  We construct the rectifier layer in such a way that it divides
 the input space into four `square' cones; each of them corresponding to the
 inputs where two of the rectifier units are active.  We define the four
 rectifiers as:
 \begin{align*}
     f_1 (\vx) =& \max\left\{0, \left[ 1, 0 \right]^\top \vx\right\}, \\
     f_2 (\vx) =& \max\left\{0, \left[-1, 0 \right]^\top \vx\right\}, \\
     f_3 (\vx) =& \max\left\{0, \left[ 0, 1 \right]^\top \vx\right\}, \\
     f_4 (\vx) =& \max\left\{0, \left[ 0,-1 \right]^\top \vx\right\},
 \end{align*}
 where $\vx = \left[ x_1, x_2\right]^\top \in \RR^{n_0}$. 
 By adding pairs of coordinates of $\vf = \left[f_1,
 f_2, f_3, f_4 \right]^\top$, we can effectively mimic a layer
 consisting of two absolute-value units $g_1$ and $g_2$:
 \begin{align}
 \label{eq:comb1}
     \left[ 
         \begin{array}{c}
             g_1(\vx) \\
             g_2(\vx)
         \end{array}
     \right]
     = 
     \left[
         \begin{array}{c c c c}
             1 & 1 & 0 & 0 \\
             0 & 0 & 1 & 1
         \end{array}
     \right]
     \left[ 
         \begin{array}{c}
             f_1(\vx) \\
             f_2(\vx) \\ 
             f_3(\vx) \\ 
             f_4(\vx) 
         \end{array}
     \right]
     =
     \left[
         \begin{array}{c}
             \abs(x_1) \\
             \abs(x_2)
         \end{array}
     \right]. 
 \end{align}
 
The absolute-value unit $g_i$ divides the input space along the $i$-th coordinate axis, taking values which are symmetric about that axis. 
The combination of $g_1$ and $g_2$ is then a function with four regions of linearity; 
 \begin{align*}
     \Scal_1 =& \left\{ (x_1, x_2) \mid x_1 \geq 0, x_2 \geq 0 \right\} \\
     \Scal_2 =& \left\{ (x_1, x_2) \mid x_1 \geq 0, x_2 < 0 \right\} \\
     \Scal_3 =& \left\{ (x_1, x_2) \mid x_1 < 0, x_2 \geq 0 \right\} \\
     \Scal_4 =& \left\{ (x_1, x_2) \mid x_1 < 0, x_2 < 0 \right\}.
 \end{align*}
 
Since the values of $g_i$ are symmetric about the $i$-th coordinate axis, 
each point $\vx \in \Scal_i$ has a corresponding point $\vy \in \Scal_j$ with
$\vg(\vx)= \vg(\vy)$, for all $i$ and $j$.

 \begin{figure}[ht]
     \centering
     \begin{minipage}[t]{0.49\textwidth}
         \centering
         \includegraphics[width=0.9\columnwidth]{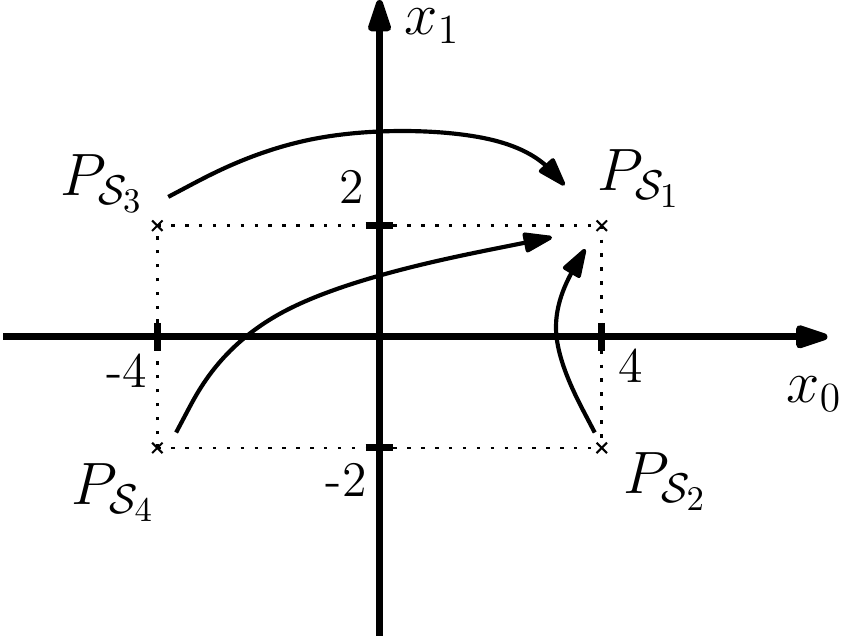}
 
         (a) 
     \end{minipage}
     \hfill
     \begin{minipage}[t]{0.49\textwidth}
         \centering
         \includegraphics[width=0.9\columnwidth]{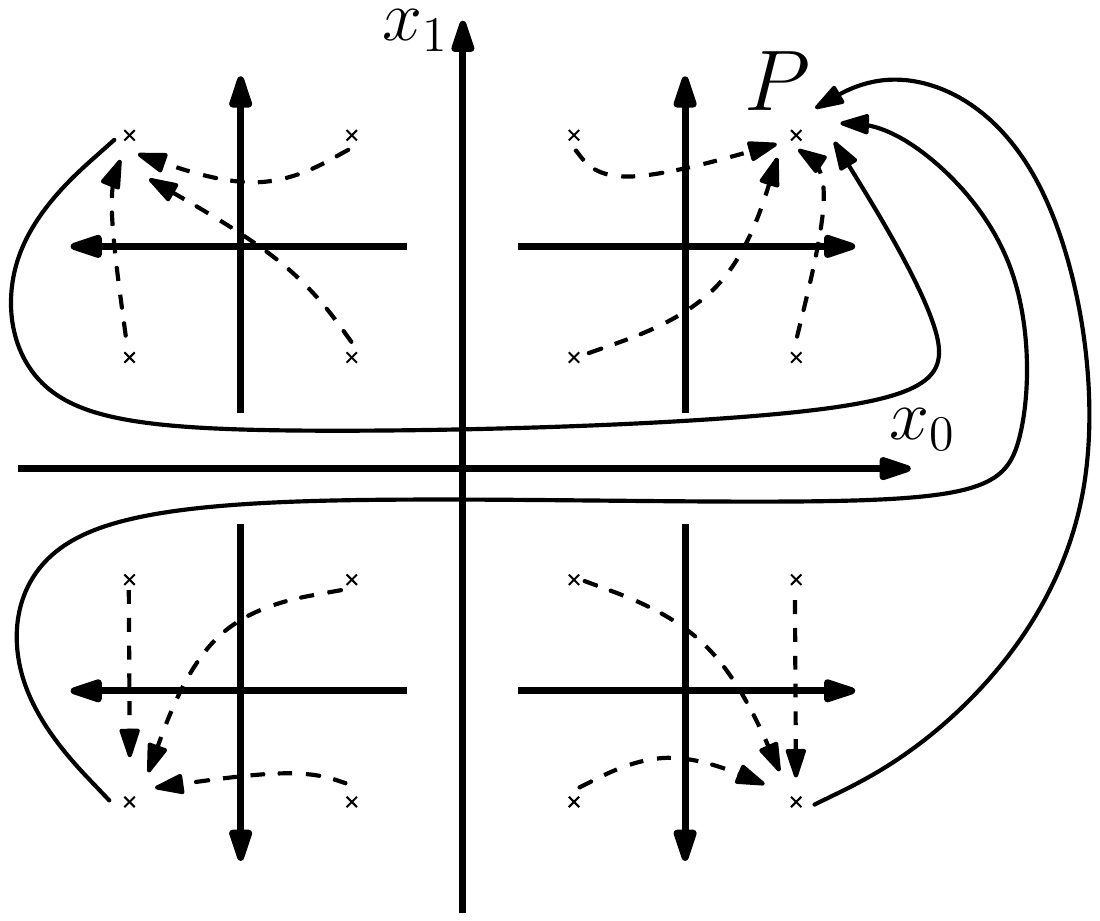}
 
         (b)
     \end{minipage}
     \caption{Illustration of Example~\ref{example:illustration}. 
     (a) A rectifier layer with two pairs of units, where each pair computes the absolute value of one of two input coordinates. 
         Each input quadrant is mapped to the positive quadrant. 
         (b) Depiction of a two layer model. Both layers
         simulate the absolute value of their input coordinates. 
}
     \label{fig:simple}
 \end{figure}
 
 We can apply the same procedure to the image of $\left[ g_1, g_2 \right]$ to
 recursively divide the input space, as illustrated in Fig.~\ref{fig:simple}.
 For instance, if we apply this procedure one more time, we get four regions
 within each $\Scal_i$, resulting in $16$ regions in total, within the input
 space.  On the last layer, we may place rectifiers in any way suitable for the
 task of interest (e.g., classification).  The partition computed by the last
 layer will be copied to each of the input space regions that produced the same
 input for the last layer.  Fig.~\ref{fig:unit_cube} shows a function that can
 be implemented efficiently by a deep model using the previous observations. 
 
 \begin{figure}[ht]
     \centering
     \hspace{-.03\textwidth}
     \begin{minipage}[t]{0.32\textwidth}
         \centering
         \includegraphics[clip=true, trim=17cm 4cm 17cm 4cm, width=1.\textwidth]{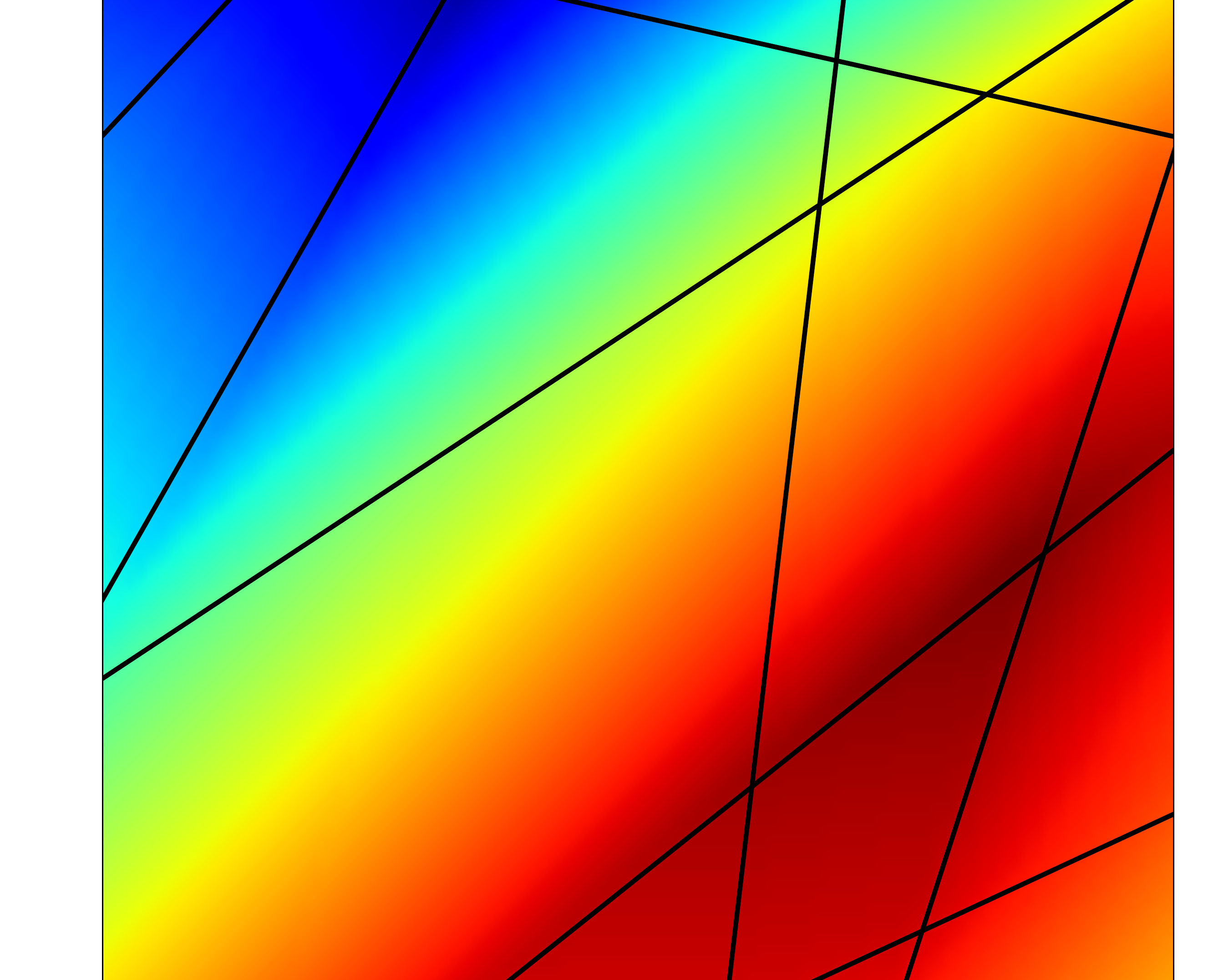}
         
         (a) 
     \end{minipage}
     \begin{minipage}[t]{0.32\textwidth}
         \centering
         \includegraphics[clip=true, trim=17cm 4cm 17cm 4cm, width=1.\textwidth]{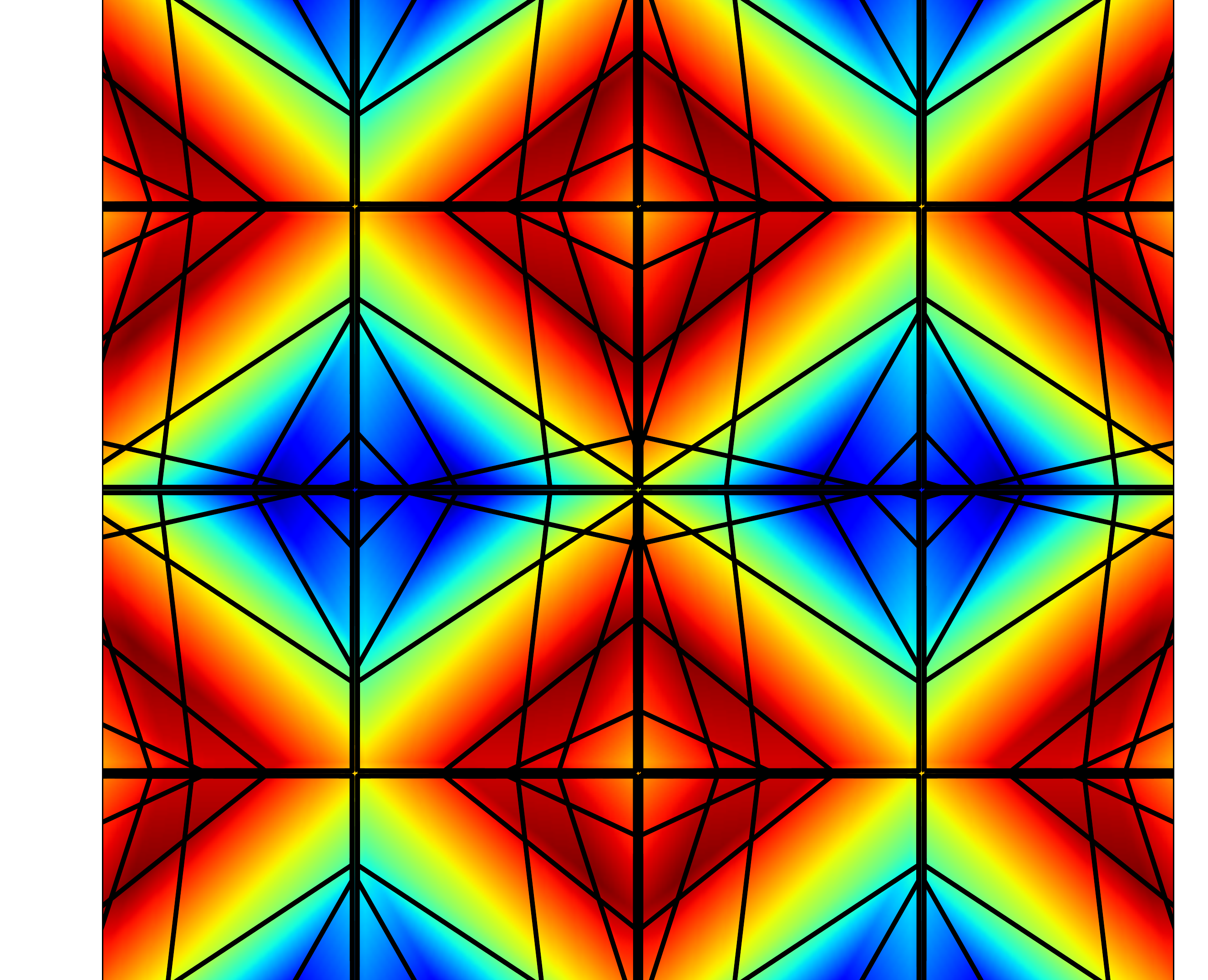}
         
         (b)
     \end{minipage}
     \begin{minipage}[t]{0.32\textwidth}
         \centering
         \includegraphics[clip=true, trim=17cm 4cm 17cm 4cm, width=1.\textwidth]{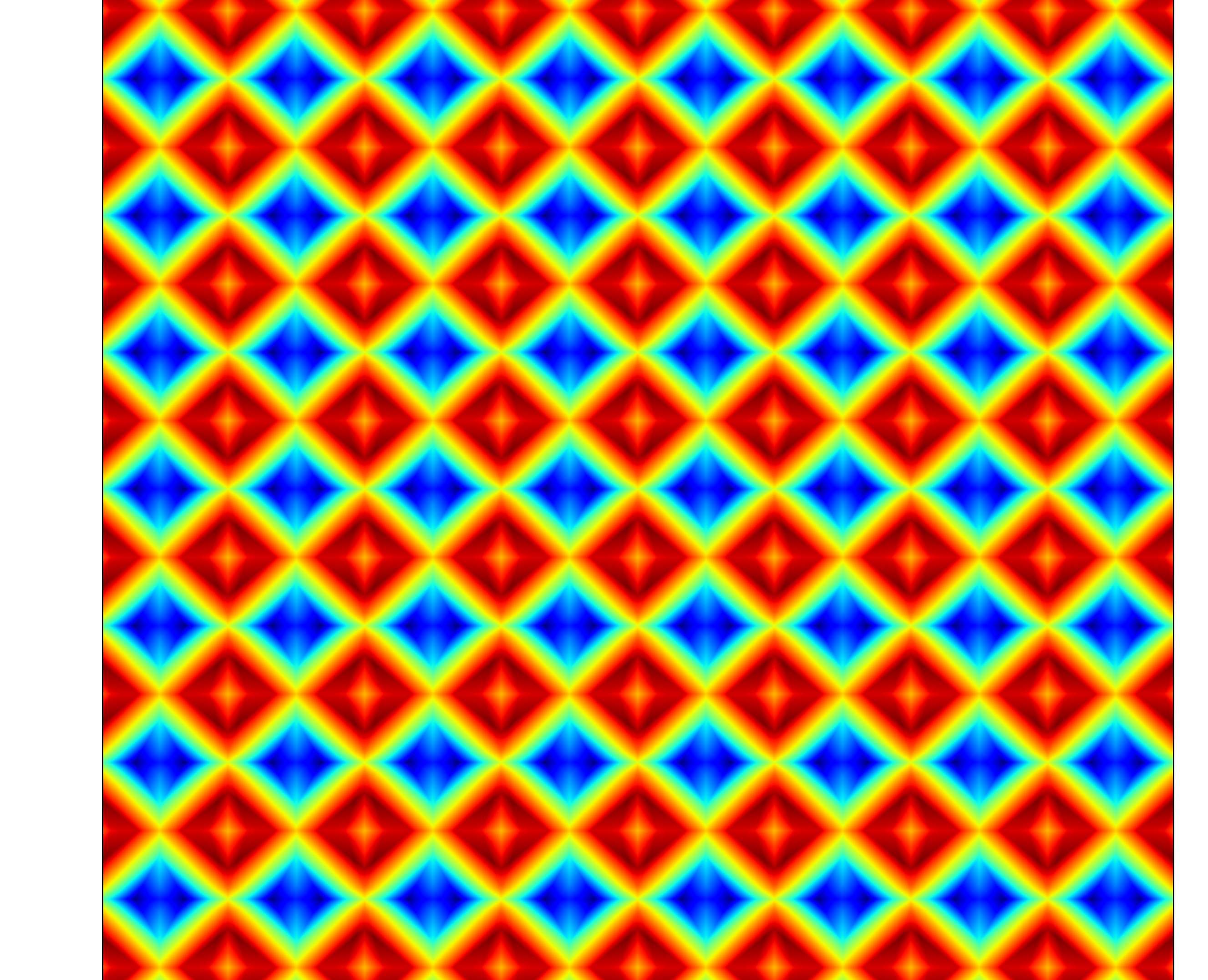}
         
         (c)
     \end{minipage}
     \caption{(a)
     Illustration of the partition computed by $8$ rectifier units on the outputs $(x_1,x_2)$ of the preceding layer. 
     The color is a heat map of $x_1 - x_2$. 
      (b) 
      Heat map of a function computed by a rectifier network with $2$ inputs,  $2$ hidden layers of width $4$, and one linear output unit. 
      The black lines delimit the regions of linearity of the function. 
      (c) 
     Heat map of a function computed by a $4$ layer model with a total of $24$ hidden units.   
     It takes at least $137$ hidden units on a shallow model to represent the same function.}
     \label{fig:unit_cube}
 \end{figure}
 
 \end{example}

 The foregoing discussion can be easily generalized to $n_0 > 2$ input variables and $k$ hidden layers, each consisting of $2n_0$ rectifiers. 
 In that case, the maximal number of linear regions of functions computable by the network is lower-bounded as follows. 
 
 \begin{theorem} 
     \label{thm:opt1}
     The maximal number of regions of linearity of functions computable by a
     rectifier neural network with $n_0$ input variables and $k$ hidden layers
     of width $2n_0$ is at least $2^{(k-1)n_0} \sum_{j=0}^{n_0} { 2n_0 \choose j }$.
 \end{theorem}
 
 \begin{proof}
 We prove this constructively.  We define the rectifier units in each hidden
 layer in pairs, with the sum of each pair giving the absolute value of a
 coordinate axis.  We interpret also the sum of such pairs as the actual input
 coordinates of the subsequent hidden layers.  The rectifiers in the first
 hidden layer are defined in pairs, such that the sum of each pair is the
 absolute value of one of the input dimensions, with bias equal to
 $(-\tfrac12,\ldots,-\tfrac12)$.  In the next hidden layers, the rectifiers are
 defined in a similar way, with the difference that each pair computes the
 absolute value of the sum of two of their inputs.  The last hidden layer is
 defined in such a way that it computes a piece-wise linear function with the
 maximal number of pieces, all of them intersecting the unit cube in
 $\mathbb{R}^{n_0}$.  The maximal number of regions of linearity of $m$
 rectifier units with $n_0$-dimensional input is $\sum_{j=0}^{n_0}{m\choose
 j}$.  This partition is multiplied in each previous layer $2^{n_0}$ times.
 \end{proof}
 
 The theorem shows that even for a small number of layers $k$, we can have many
 more linear regions in a deep model than in a shallow one.  For example, if we
 set the input dimensionality to $n_0=2$, a shallow model with $4n_0$ units
 will have at most $37$ linear regions.  The equivalent deep model with two
 layers of $2n_0$ units can produce $44$ linear regions.  For $6n_0$ hidden
 units the shallow model computes at most $79$ regions, while the equivalent
 three layer model can compute $176$ regions.

\section{Discussion and conclusions}
\label{section:conclusions}

In this paper we introduced a novel way of understanding the expressiveness of
neural networks with piecewise linear activations.  We count the number of
regions of linearity, also called response regions, of the functions that they
can represent.  The number of response regions tells us how well the models can
approximate arbitrary curved shapes.  Computational Geometry provides us the
tool to make such statements. 

We found that deep and narrow rectifier MPLs can generate many more regions of
linearity than their shallow counterparts with the same number of computational
units or of parameters. 
We can express this in terms of the ratio between the maximal number of
response regions and the number of parameters of both model classes.  For a
deep model with $n_0=O(1)$ inputs and $k$ hidden layers of width $n$, the
maximal number of response regions per parameter behaves as
$$\Omega\left(\left\lfloor\frac{n}{n_0}\right\rfloor^{k-1}
\frac{n^{n_0-2}}{k}\right).$$ For a shallow model with $n_0=O(1)$ inputs, the
maximal number of response regions per parameter behaves as $$O\left(k^{n_0-1}
n^{n_0-1}\right).$$ We see that the deep model can generate many more response
regions per parameter than the shallow model; exponentially more regions per
parameter in terms of the number of hidden layers $k$, and at least order
$(k-2)$ polynomially more regions per parameter in terms of the layer width
$n$.  In particular, there are deep models which use fewer parameters to
produce more linear regions than their shallow counterparts.  Details about the
asymptotic expansions are given in the Appendix~\ref{section:asymptotic}.

In this paper we only considered linear output units, but this is not a
restriction, as the output activation itself is not parametrized.  If there is
a target function $f_{\text{targ}}$ that we want to model with a rectifier MLP
with $\sigma$ as its output activation function, then there exists a function
$f'_{\text{targ}}$ such that $\sigma(f'_{\text{targ}}) = f_{\text{targ}}$, when
$\sigma$ has an inverse (e.g., with sigmoid), $f'_{\text{targ}} =
\sigma^{-1}(f_{\text{targ}})$.  For activations that do not have an inverse,
like softmax, there are infinitely many functions $f'_{\text{targ}}$ that work.
We just need to pick one, e.g., for softmax we can pick
$\log(f_{\text{targ}})$.  By analyzing how well we can model $f'_{\text{targ}}$
with a linear output rectifier MLP we get an indirect measure of how well we
can model $f_{\text{targ}}$ with an MLP that has $\sigma$ as its output
activation.

Another interesting observation is that we recover a high ratio of $n$ to $n_0$
if the data lives near a low-dimensional manifold (effectively like reducing
the input size $n_0$).  One-layer models can reach the upper bound of response
regions only by spanning all the dimensions of the input.  In other words,
shallow models are not capable of concentrating linear response regions in any
lower dimensional subspace of the input. 
If, as commonly assumed, data lives near a low dimensional manifold, then we
care only about the number of response regions that a model can generate in the
directions of the data manifold.  One way of thinking about this is principal
component analysis (PCA), where one finds that only few input space directions
(say on the MNIST database) are relevant to the underlying data.  In such a
situation, one cares about the number of response regions that a model can
generate only within the directions in which the data does change. In such
situations $n \gg n_0$, and our results show a clear advantage of using deep
models. 

We believe that the proposed framework can be used to answer many other
interesting questions about these models. For example, one can look at how the
number of response regions is affected by different constraints of the model,
like shared weights.  We think that this approach can also be used to study
other kinds of piecewise linear models, such as convolutional networks with
rectifier units or maxout networks, or also for comparing between different
piecewise linear models.

\appendix

\section{Asymptotic}
 \label{section:asymptotic}
 
Here we derive asymptotic expressions of the formulas contained in
Proposition~\ref{proposition:onelayermodel} and
Theorem~\ref{theorem:klayermodel}.  We use following standard notation: 
 \begin{itemize}
  \item 
  $f(n) = O(g(n)) $ means that there is a positive constant $c_2$ such that $
  f(n) \leq c_2 g(n)$ for all $n$ larger than some $N$.  
 \item 
 $f(n) = \Theta(g(n)) $ means that there are two positive constants $c_1$ and
 $c_2$ such that $c_1 g(n) \leq f(n) \leq c_2 g(n)$ for all $n$ larger than
 some $N$. 
 \item 
 $f(n) = \Omega(g(n)) $ means that there is a positive constant $c_1$ such that
 $ f(n) \geq c_1 g(n)$ for all $n$ larger than some $N$.  
\end{itemize}

\begin{proposition}\mbox{}
\label{prop:bounds}
\begin{itemize}
\item 
Consider a single layer rectified MLP with $kn$ units and $n_0$ inputs.  Then
the maximal number of regions of linearity of the functions represented by this
network is 
\begin{equation*}
\mathcal{R} (n_0,kn,1) =  \sum_{s=0}^{n_0} {{ kn} \choose s} ,
\end{equation*}
and 
\begin{equation*}
    \mathcal{R} (n_0,kn,1) = O(k^{n_0}n^{n_0}),\quad\text{when $n_0=O(1)$.}
\end{equation*}
\item 
Consider a k layer rectified MLP with hidden layers of width $n$ and $n_0$
inputs.  Then the maximal number of regions of linearity of the functions
represented by this network satisfies 
\begin{equation*}
\mathcal{R} (n_0,n,\ldots,n,1)  \geq  
\left(
\prod_{i=1}^{k-1}
\left\lfloor \frac{n}{n_0}\right\rfloor
\right) 
\sum_{s=0}^{n_0} {n \choose s},  
\end{equation*}
and 
\begin{equation*}
    \mathcal{R} (n_0,n,\ldots,n,1)  = \Omega\left( \left\lfloor\frac{n}{n_0}\right\rfloor^{k-1}n^{n_0} \right),\quad\text{when $n_0=O(1)$}.
\end{equation*}
\end{itemize}
\end{proposition}

\begin{proof}
Here only the asymptotic expressions remain to be shown. 
It is known that 
\begin{equation}
\sum_{s=0}^{n_0} { m \choose s} = \Theta\left( \left(1-\frac{2n_0}{m}\right)^{-1} {m\choose n_0} \right),\quad \text{when $n_0\leq \frac{m}{2}-\sqrt{m}$}. 
\end{equation}

Furthermore, it is known that 
\begin{equation}
{m\choose s} = \frac{m^s}{s!}\left( 1+ O(\tfrac{1}{m}) \right),\quad\text{when $s=O(1)$}. 
\end{equation}

When $n_0$ is constant, $n_0=O(1)$, we have that 
\begin{equation*}
{kn\choose n_0} = 
 \frac{k^{n_0}}{n_0!}n^{n_0} \left(1+ O\left(\tfrac{1}{kn}\right) \right) . 
\end{equation*}
In this case, it follows that 
\begin{equation*}
 \sum_{s=0}^{n_0} { k n \choose s} = \Theta\left( \left(1-\frac{ 2 n_0}{kn}\right)^{-1} {kn\choose n_0} \right) = \Theta\left( k^{n_0} n^{n_0}  \right)  
\quad \text{and also }\quad 
\sum_{s=0}^{n_0} {  n \choose s}  
 =  \Theta(n^{n_0}). 
\end{equation*}
Furthermore,
\begin{align*}
    \left(\prod_{i=1}^{k-1}\left\lfloor\frac{n}{n_0}\right\rfloor\right)  \sum_{s=0}^{n_0} {  n \choose s}   
    &=  \Theta\left(\left\lfloor\frac{ n}{n_0 }\right\rfloor^{k-1}   n^{n_0}\right). 
\qedhere
\end{align*}
\end{proof}

We now analyze the number of response regions as a function of the number of parameters. 
When $k$ and $n_0$ are fixed, then $\left\lfloor {n}/{n_0}\right\rfloor^{k-1}$ grows polynomially in $n$, and $k^{n_0}$ is constant. 
On the other hand, when $n$ is fixed with $n > 2n_0$, then $\left\lfloor {n}/{n_0}\right\rfloor^{k-1}$ 
grows exponentially in $k$, and $k^{n_0}$ grows polynomially in $k$. 

\begin{proposition}
\label{prop:number_params}
The number of parameters of a deep model with 
$n_0 = O(1)$ inputs, $n_{\operatorname{out}}=O(1)$ outputs, and $k$ hidden layers of width $n$ is 
$$(k-1) n^2 + (k + n_0 + n_{\operatorname{out}}) n + n_{\operatorname{out}} = O(k n^2).$$ 
The number of parameters of a shallow model with $n_0 = O(1)$ inputs, $n_{\operatorname{out}} = O(1)$ outputs, and  
$k n$ hidden units is
$$ (n_0 + n_{\operatorname{out}}) k n + n + n_{\operatorname{out}} = O(k n).$$  
\end{proposition}

\begin{proof}
 For the deep model, each layer, except the first and last, 
 has an input weight matrix with $n^2$ entries and a bias vector of length $n$. 
 This gives a total of $(k-1) n^2 + (k-1) n$ parameters. 
 The first layer has $n n_0$ input weights  and $n$ bias. 
 The output layer has $n n_{out}$ input weight matrix and $n_{\operatorname{out}}$ bias. 
 If we sum these together we get $$ (k-1) n^2 + n (k + n_0 + n_{\operatorname{out}}) + n_{\operatorname{out}} = O( k n^2).$$  

 For the shallow model, the hidden layer has $k n n_0$ input weights and $k n$ bias. 
 The output weights has $k n n_{\operatorname{out}}$ input weights and $n_{\operatorname{out}}$ bias. 
 Summing these together we get
$$ k n (n_0 + n_{\operatorname{out}}) + n + n_{\operatorname{out}} = O (k n). $$
\end{proof}

The number of linear regions per parameter can be given as follows. 

\begin{proposition}
 Consider a fixed number of inputs $n_0$ and a fixed number of outputs $n_{\operatorname{out}}$. 
 The maximal ratio of the number of response regions to the number of parameters 
 of a deep model with $k$ layers of width $n$ is 
$$\Omega\left(\left\lfloor\frac{n}{n_0}\right\rfloor^{k-1} \frac{n^{n_0-2}}{k}\right).$$
In the case of a shallow model with $kn$ hidden units, the ratio is 
$$O\left(k^{n_0-1} n^{n_0-1}\right).$$
\end{proposition}

\begin{proof}
 This is by combining Proposition~\ref{prop:bounds} and Proposition~\ref{prop:number_params}.   
\end{proof}

We see that fixing the number of parameters, deep models can compute functions
with many more regions of linearity that those computable by shallow models.
The ratio is exponential in the number of hidden layers $k$ and thus in the
number of hidden units. 

{\small
\subsubsection*{Acknowledgments}
We would like to thank KyungHyun Cho, \c{C}a\u{g}lar G\"{u}l\c{c}ehre, and anonymous ICLR reviewers 
for their comments. 
Razvan Pascanu is supported by a DeepMind Fellowship.

\bibliographystyle{abbrvnat}
\bibliography{references}{}

\end{document}